\documentclass{article}
\usepackage{graphicx} 
\usepackage{amsfonts}
\usepackage{amsmath}
\usepackage{stackengine}
\usepackage{xcolor}
\usepackage[all]{xy}
\usepackage{mathtools}
\usepackage{amsthm}
\usepackage{todonotes}
\usepackage{enumerate}
\usepackage{amssymb}
\usepackage{tikz-cd}
\usepackage{listings}
\usepackage{comment}
\usepackage{subcaption}
\usepackage{stackengine}
\usepackage{algorithmic}

\theoremstyle{plain}
\newtheorem{theorem}{Theorem}
\newtheorem{conjecture}{Conjecture}
\newtheorem{definition}{Definition}[section]
\newtheorem{lemma}[definition]{Lemma}
\newtheorem{corollary}[definition]{Corollary}
\newtheorem{proposition}[definition]{Proposition}

\newtheorem*{theorem*}{Theorem}
\newtheorem*{theoremA}{Theorem~\ref{thm:supfundimbounds}}
\newtheorem*{conjectureA}
{Conjecture~\ref{c:pPdimequalsfundim}}
\theoremstyle{definition}
\newtheorem{example}[definition]{Example}

\newtheorem{question}[definition]{Question}

\theoremstyle{remark}
\newtheorem{remark}[definition]{Remark}

\newcommand{\DTheta}{\mathcal{D}_{\theta_0}}
\newcommand{\bFD}{dim_{ba.fun}}
\newcommand{\FD}{dim_{fun}}
\newcommand{\DepsilonTheta}{\mathcal{D}_{\theta_0,\epsilon}}
\newcommand{\VCD}{dim_{VC}}
\newcommand{\pVCD}{dim_{VC\Delta}}
\newcommand{\ppVCD}{dim_{p.{VC\Delta}}}

\title{On Functional Dimension and Persistent Pseudodimension}
\author{J. Elisenda Grigsby\footnote{Equal contribution} \ and Kathryn Lindsey\footnotemark[1]\\
Department of Mathematics, Boston College\\
\texttt{grigsbyj@bc.edu, lindseka@bc.edu}
}

\date{\today}

\begin{document}

\maketitle

\begin{abstract}
For any fixed feedforward ReLU neural network architecture, it is well-known that many different parameter settings can determine the same function. It is less well-known that the degree of this redundancy is inhomogeneous across parameter space. In this work, we discuss two locally-applicable complexity measures for ReLU network classes and what we know about the relationship between them: (1) the local functional dimension \cite{GLMW,BP1}, and (2) a local version of VC dimension that we call {\em persistent pseudodimension}. The former is easy to compute on finite batches of points; the latter should give local bounds on the generalization gap, which would inform an understanding of the mechanics of the double descent phenomenon \cite{Belkin}. \end{abstract}

\section{Introduction}
Recall that a {\em labeled training set} for a regression task is a finite set \[\mathcal{D} = \left\{\left(x^{(i)},y^{(i)}\right) \in \mathbb{R}^{n_0} \times \mathbb{R}^{n_d} \right\}_{i=1}^m\] of pairs drawn i.i.d. from an unknown data-generating probability distribution \[\mathcal{P}: \mathbb{R}^{n_0} \times \mathbb{R}^{n_d} \rightarrow [0, \infty).\]
The goal of any supervised learning algorithm is to effectively predict the outputs from the inputs on {\em unseen} data drawn from $\mathcal{P}$. One typically does this by running an optimization algorithm on the parameter space, $\Omega$, of a fixed parameterized function class $\mathcal{F}: \Omega \times \mathbb{R}^{n_0} \rightarrow \mathbb{R}^{n_d}$ to find a function $F_\theta$ minimizing the empirical mean squared error on a training set $\mathcal{D}$:
	\[\mathcal{L}_\mathcal{D}(\theta) := \frac{1}{2m} \sum_{i=1}^m \left(F_\theta(x^{(i)}) - y^{(i)}\right)^2.\] The classical bias-variance trade-off suggests that  interpolating parameters, i.e., those $\theta \in \Omega$ for which $\mathcal{L}_\mathcal{D}(\theta) = 0$, should correspond to functions that overfit $\mathcal{D}$, hence generalize poorly. But it has been observed empirically and proved theoretically in certain settings that 
many modern heavily-overparameterized function classes admit high-dimensional spaces of interpolating solutions, and moreover those solutions found by optimization algorithms like stochastic gradient descent perform well on test data unobserved during training. Why?

Any hypothesis one might use to explain the observed phenomena requires that the generalization behavior of an overparameterized function class is inhomogeneous within the class, and that stochastic gradient descent has an implicit bias towards functions that generalize better, cf. \cite{Belkin, Cooper}, the introductory sections of \cite{BMR,MM,BP2}, and \cite{BGKP} for a recent survey, which includes a bibilography of theoretical and experimental results in this direction. 

Note that this circle of ideas is closely connected to the questions of {\em redundancy} and {\em symmetry} in modern parameterized function classes. Unlike classical function classes like polynomials, modern parameterized function classes such as deep neural networks with ReLU activation do not admit a one-to-one correspondence between parameter settings and functions, cf. \cite{RolnickKording, PhuongLampert}.

In \cite{GLMW}, we and collaborators defined a local notion of complexity for parameterized ReLU neural network function classes called {\em functional dimension}.  It is a notion of complexity assigned to each {\em parameter} $\theta$ within a function class, and not to the function class itself. We denote it by \[\FD(\theta).\] This notion was independently discovered and studied in \cite{BP1, BP2}. In the latter work, the authors moreover proved that the local functional dimension gives a lower bound for a local version of the {\em fat-shattering dimension}, a classical measure of generalization behavior that one can view as a variant of the VC dimension. In \cite{GLR}, we and collaborators empirically investigated the distribution of (an approximation of) functional dimension, and found that this approximation is indeed inhomogeneous across parameter space. That is, different parameters have different local functional dimensions, and the functional dimension appears to be a better measure of complexity than the parametric dimension.

The primary purpose of the present work is to relate functional dimension to another refined local version of VC dimension for piecewise polynomial function classes, which we call the {\em persistent pseudodimension}.  Note that -- as is the case for the local functional dimension defined in \cite{GLMW} -- the persistent pseudodimension is an invariant of a {\em parameter} $\theta$ in a function class and not of the function class itself. We denote it by \[\ppVCD(\theta).\]We believe this notion is natural, and our goal in the sequel is to establish generalization bounds that are inhomogeneous in parameter space.

As in \cite{BP2}, we prove that the local functional dimension is a lower bound on the persistent pseudodimension.  Here we take the additional step of establishing explicit connections between the persistent pseudodimension and two different notions of rank for matrices over the polynomial ring $R = \mathbb{R}[\vec{\theta}]$ of formal parameters for the  function class: the row rank with respect to $R$, and the row rank with respect to $\mathbb{R}$. This allows us to establish both lower and upper bounds on the persistent pseudodimension of every parameter $\theta \in \Omega$ away from a Lebesgue measure zero set: 

\begin{theoremA}
    Let $\mathcal{F}$ be a  parameterized family of ReLU neural network functions of fixed architecture and output dimension $1$ and parameter space $\Omega = \mathbb{R}^D$. For almost all $\theta \in \Omega$, \[\FD(\theta) \leq \ppVCD(\theta) \leq \sup_{Z \in {\bf Z}} \left\{r_\mathbb{R}({\bf J}E^R_Z(\theta))\right\} .\]
\end{theoremA}

\noindent Here, $E^R_Z(\theta)$ is the \emph{algebraic evaluation map} (Def. \ref{def:algebraicversions}), which outputs a vector whose $i$th entry is the polynomial expression for $F_\theta(z_i)$ (as a function of $\theta$), and $\mathbf{J}$ denotes taking the Jacobian derivative of this map;  thus,  ${\bf J}E^R_Z(\theta))$ is a matrix whose $i$th row is the gradient with respect to the parameters.  Applying $r_\mathbb{R}$ computes the real row rank (Def. \ref{def:realRowRank}) of this matrix, and the supremum is taken over all finite batches of points $Z = \{z_1, \ldots, z_m \}$ in the domain.

Since the polynomials for the piecewise polynomial class of ReLU neural network functions have combinatorial interpretations in terms of paths in a computational graph, there are serious constraints on the types of $\mathbb{R}$--linear dependence relations that can arise among the rows of the Jacobian. Indeed, we prove (Proposition \ref{p:multirankupperbound}) that for almost all batches $Z$ and parameters $\theta$: \[r_\mathbb{R}({\bf J}E^R_Z(\theta)) \leq \mbox{rk}(\alpha(\theta,Z)).\] The matrix $\alpha(\theta,Z)$ is  called the {\em activation matrix} for $\theta, Z$. It was originally defined in \cite{GS,BP1} and incorporates information about the active paths in the computational graph for the architecture in a very natural way. See Sections \ref{sec:CompGraph} and \ref{sec:ActMatrix} for more details.

These considerations, taken together, lead us to conjecture that the upper and lower bounds on the persistent pseudodimension of a parameter in a ReLU neural network class are equal.

\begin{conjectureA} Let $\mathcal{F}$ be a parameterized family of ReLU neural network functions of fixed architecture, output dimension $1$, and paramater space $\Omega = \mathbb{R}^D$. For almost all $\theta \in \Omega$,
\[\mbox{dim}_{fun}(\theta) = \ppVCD(\mathcal{F},\theta).\] 
\end{conjectureA}

A key step in proving the inequality $\textrm{dim}_{\textrm{fun}}(\theta) \leq \ppVCD(\theta)$ of Theorem \ref{thm:supfundimbounds} is understanding how the parameter space is structured locally, in terms of sets of parameters that determine the same function on a finite batch of inputs. We prove that for almost all parameters, these \emph{batch fibers} look as expected:

\begin{theorem*}[Informal version of Theorem \ref{t:batchfiberfoliation}, Batch Fiber Product Structure]
Fix a finite set $Z$ of points in input space $\mathbb{R}^{n_0}$.  For any architecture of feedforward ReLU neural networks $\mathbb{R}^{n_0} \to \mathbb{R}$, locally near almost every parameter $\theta$, the parameter space $\Omega \cong \mathbb{R}^D$ is (diffeomorphic to) a product of a $\textrm{dim}_{\textrm{ba.fun}}(\theta,Z)$-dimensional submanifold consisting of parameters that all determine different functions on $Z$ and a $(D- \textrm{dim}_{\textrm{ba.fun}}(\theta,Z))$-dimensional subspace consisting of directions in which perturbing the parameter does not change the function on $Z$. 
\end{theorem*}
\noindent (Here $\textrm{dim}_{\textrm{ba.fun}}$ denotes \emph{batch functional dimension}, see Def. \ref{def:batchfundim}.)   A consequence of this theorem is that locally near a full measure set of parameters $\theta$, level sets of the empirical loss landscape with respect to the training set $Z$ are unions of these $(D- \textrm{dim}_{\textrm{fun}}(\theta))$-dimensional submanifolds, and gradient vectors for the loss function are constrained to lie in directions orthogonal to these submanifolds. The batch functional dimension can therefore be viewed as the local number of {\em effective trainable parameters} for the batch. 

We note that a special case of the theorem above is when the batch $Z$ is ``big enough'' and ``well-chosen enough'' that specifying the values of $F_{\theta}(z_i)$ uniquely determines the function $F_\theta$ (at least locally, meaning that there is some open neighborhood $U$ of $\theta$ such that for any parameter $\theta' \in U$, $F_{\theta'}(z_i) = F_{\theta}(z_i)$ for all $\theta$ implies $\theta' = \theta$).  In this case,  $\textrm{dim}_{\textrm{ba.fun}}(\theta,Z) = \textrm{dim}_{\textrm{fun}}(\theta)$, and we say that the functional dimension is attained on the batch $Z$.  
Proposition 5.16 of \cite{GLMW} proved that for a generic, transversal, combinatorially stable parameter $\theta$, functional dimension is attained on any \emph{decisive set} $Z$, and gave an explicit definition of a decisive set for such a parameter.

The paper is organized as follows. 
\begin{itemize}
    \item In Section 2, we establish background on parameterized families of functions, focusing on feedforward ReLU neural network families and other piecewise-polynomial families. In Section \ref{sec:CompGraph} we remind the reader how to understand the piecewise-polynomial structure of a ReLU network class via weighted paths in a computational graph. In Section \ref{sec:Rkpoly}, we recall some classical algebraic results about matrices over polynomial rings.
    \item In Section 3, we recall the classical notions of VC dimension and pseudodimension for parameterized function classes and then define the persistent pseudodimension. We close the section by proving a local version of the Perles-Sauer-Shelah lemma, as well as the important technical result (Proposition \ref{prop:ParSmooth}) that a batch on which the persistent pseudodimension is realized can be chosen so that the Jacobian matrix is well-defined on the batch.
    \item In Section \ref{sec:BatchFiber}, we prove Theorem \ref{t:batchfiberfoliation}.
    \item In Section \ref{sec:RankGap}, we establish  upper and lower bounds (Theorem \ref{t:inequalitystring}) on the batch persistent pseudodimension, in terms of the row ranks of the Jacobian matrix with respect to $R=\mathbb{R}[\vec{\theta}]$ and $\mathbb{R}$, as  introduced in Section \ref{sec:Rkpoly}. The lower bound agrees with the batch functional dimension for almost all parameters.
    \item In Section \ref{sec:mainthm} we take suprema over finite batches to prove Theorem \ref{thm:supfundimbounds}.
    \item In Section \ref{sec:ActMatrix} we describe the activation matrix associated to a parameter and prove that its rank is an upper bound for all other ranks described in this paper.
\end{itemize}

\subsection*{Acknowledgments}
Many thanks to Adam Klivans, Fran{c}ois Malgouyres, Joachim Bona-Pellissier, and Matus Telgarsky for interesting conversations. We also thank Francois Malgouyres for valuable comments on an earlier version of the paper. JEG and KL were partially supported by NSF grant \#2133822 through the SCALE MoDL program. JEG was partially supported by Simons grant \#635578.
\section{Parameterized families and associated maps}

\subsection{Preliminary definitions}
\begin{definition} [parameterized family] \label{def:parameterizedfamily}
A {\em parameterized family} is a tuple $(\Omega, \mathbb{R}^{n_0}, \mathbb{R}^{n_d}, \mathcal{F})$ where
\begin{itemize}
    \item $\Omega = \mathbb{R}^D$ is the {\em parameter space},
    \item $\mathbb{R}^{n_0}$ is the {\em input space} or {\em domain},
    \item $\mathbb{R}^{n_d}$ is the {\em output space} or {\em codomain},
    \item $\mathcal{F}: \Omega \times \mathbb{R}^{n_0} \rightarrow \mathbb{R}^{n_d}$ is a continuous function, 
\end{itemize}
\end{definition}

\noindent Throughout the paper, for the sake of notational efficiency, we will use $\mathcal{F}$ to refer to the parameterized family $(\Omega, \mathbb{R}^{n_0}, \mathbb{R}^{n_d}, \mathcal{F})$.

\medskip

The definition below gives several useful maps associated to any parameterized family, and the notation we will use for these maps. (In it, we use the notation $C^0(A,B)$ to denote the collection of continuous maps from a topological space $A$ to a topological space $B$.)  

\begin{definition} [function, realization map, evaluation map]
Let $\mathcal{F}$ be a parameterized family. 
\begin{itemize}
\item For each parameter $\theta \in \Omega$, the \emph{function}  
associated to $\theta$ is the map $F_\theta : \mathbb{R}^{n_0} \to \mathbb{R}^{n_d}$ given by $F_\theta(x) =  \mathcal{F}(\theta,x)$. 
\item The \emph{realization map} is the map $\rho: \Omega \rightarrow C^0(\mathbb{R}^{n_0},\mathbb{R}^{n_d})$ given by 
  $\rho(\theta) = F_{\theta}$, i.e. the realization map assigns to each parameter $\theta \in \Omega$ the associated function 
  $F_{\theta}$. 
\item For each point $z \in \mathbb{R}^{n_0}$, the \emph{evaluation map} $E_z: \Omega \to \mathbb{R}^{n_d}$ is given by 
$E_z(\theta) = \mathcal{F}(\theta,z)$. 

\item For any finite set $Z = \{z_1,\ldots,z_m\} \subset \mathbb{R}^{n_0}$, the \emph{evaluation map}  $\mbox{E}_Z: \Omega \rightarrow (\mathbb{R}^{n_d})^m$ is given by
$E_Z(\theta) = (E_{z_1}(\theta),\ldots,E_{z_m}(\theta))$. 
\end{itemize}
\end{definition}

\begin{remark} Later, we will focus on the case where the map $E_z$ is given locally by a polynomial (whose variables are the coordinates of $\theta$); Definition \ref{def:algebraicversions}  gives a variant of the evaluation map, $E_Z^R$, which outputs formal polynomials. \end{remark}
    
We will be interested in how outputs of $\mathcal{F}$ vary under small perturbations of $\theta$; to that end, we define the parameterized family of differences.

\begin{definition}[parameterized family of differences] 
\label{def:familyofdifferences}
Let $\mathcal{F}$ be a parameterized family and fix a parameter $\theta_0 \in \Omega$. 
\begin{itemize}
    \item Define the \emph{parameterized family of differences} to be the map $\mathcal{D}_{\theta_0}: \Omega \times \mathbb{R}^{n_0} \rightarrow \mathbb{R}^{n_d}$ given by 
\[\mathcal{D}_{\theta_0}(\theta,z) = \mathcal{F}(\theta,z) -\mathcal{F}(\theta_0,z).\]
\item For $\epsilon > 0$, the \emph{parameterized family of differences $\epsilon$--close to $\theta_0$}, denoted $\DepsilonTheta$, is the restriction of $\mathcal{D}_{\theta_0}$ to $B_\epsilon(\theta_0) \times \mathbb{R}^{n_0}$, where $B_\epsilon(\theta_0)$ denotes the open ball of radius $\epsilon$ centered at $\theta_0$. 
\end{itemize}
\end{definition}

\subsection{Functional dimension definitions}

Recall from \cite[Sec. 3]{GLMW} that a {\em piece} of a Euclidean space $\mathbb{R}^n$ is a subset $X \subseteq \mathbb{R}^n$ that can be realized as the closure of a non-empty, open, connected subset of $\mathbb{R}^n$.

\begin{definition}[finitely piecewise smooth/polynomial] (cf. \cite[Defn. 3.3]{GLMW}) \label{def:piecewiseSmoothOrPoly}
A continuous function $f: \mathbb{R}^n \rightarrow \mathbb{R}^m$ is said to be {\em finitely piecewise smooth} (resp. finitely piecewise polynomial) if there exist finitely many pieces $X_1, \ldots, X_k \subseteq \mathbb{R}^{n}$ such that $\mathbb{R}^n = \bigcup_{i=1}^k X_i$ and the restriction of $f$ to the interior of each $X_i$ is smooth (resp. polynomial).   
\end{definition}

In particular, the parameterized family of feedforward ReLU neural networks (using the usual parameterization) of any fixed architecture is a finitely piecewise smooth (more specifically, finitely piecewise polynomial) parameterized function class (cf. \cite[Thm. 3.5]{GLMW}). However, all of the following definitions, made in \cite[Sec. 3]{GLMW}, apply more broadly to any finitely piecewise smooth parameterized function class.

\begin{definition} \label{def:parametricallysmooth}
  Let $\mathcal{F}$ be a finitely piecewise smooth parameterized family, and let $\theta \in \Omega$ be a parameter.  
  \begin{enumerate}[(a)]

  \item A point $z \in \mathbb{R}^{n_0}$ in the domain is said to be \emph{parametrically smooth for $\theta$} if $(\theta,z)$ is a smooth point for $\mathcal{F}$. 
  \item   A finite set $Z \subset \mathbb{R}^{n_0}$ is said to be   \emph{parametrically smooth for $\theta$} if every point $z \in Z$ is parametrically smooth  for  $\theta$.

\item A parameter $\theta \in \Omega$  is said to be \emph{ordinary} if there exists at least one point $z \in \mathbb{R}^{n_0}$ that is parametrically smooth for $\theta$.
\end{enumerate}
\end{definition}

Now assume $\mathcal{F}$ is a finitely piecewise smooth parameterized family, let $\theta \in \Omega$ be an ordinary parameter and $Z = \{z_1, \ldots, z_m\} \subset \mathbb{R}^{n_0}$ be a set of $m$ parametrically smooth points for $\theta$. Then the Jacobian matrix, ${\bf J}E_Z|_{\theta}$, is well-defined. 

\begin{definition} \label{def:batchfundim} Let $\theta \in \Omega$ be an ordinary parameter for a finitely piecewise smooth parameterized family $\mathcal{F}$.  
\begin{enumerate}[(a)]
    \item The {\em batch functional dimension} of $\theta$ for a finite batch $Z \subset \mathbb{R}^{n_0}$ of parametrically smooth points for $\theta$ is  \[\bFD(\theta,Z) := \emph{rank}\left({\bf J}E_Z\vert_\theta\right).\]
    \item The {\em (full) functional dimension} of $\theta$, denoted  $\FD(\theta)$, is the supremum of $dim_{ba.fun}(\theta, Z)$ over all finite, parameterically smooth batches $Z$.
\end{enumerate}
\end{definition} 

\begin{remark}
Denoting by $T_{{\theta}}(\Omega) \cong \mathbb{R}^D$ the tangent space to $\Omega  \cong \mathbb{R}^D$  at  the point  $\theta  \in \Omega$,   the rows of the Jacobian matrix ${\bf J}\mbox{E}_{Z}|_{\theta}$ are the gradient vectors with respect to the parameters of the output components of $\mbox{E}_{z_i}$ for $z_i \in Z$.
Accordingly, they span a subspace of $T_{\theta}(\Omega)$, which we denote by $\mathcal{T}_{\theta}(Z)$. It is immediate from the definitions that $\bFD(\theta,Z) = \mbox{dim}(\mathcal{T}_\theta(Z))$.
\end{remark}
\begin{remark}
For all ordinary $\theta$ and parametrically smooth $Z$, we have
 \[ \bFD(\theta, Z) \leq \FD(\theta) \leq D.\] 

\end{remark}

\begin{definition}
Let $\theta \in \Omega$ be an ordinary parameter. A parametrically smooth set, $Z = \{z_1, \ldots, z_k\}$, is said to be a {\em full batch} for $\theta$ if \[\bFD(\theta,Z)= \FD(\theta).\] 
\end{definition}

\subsection{ReLU neural network architectures as piecewise polynomial parameterized function classes}

In the present work, we will be most interested in parameterized families 
consisting of feedforward neural networks with ReLU activation function with fixed architecture, $(n_0,\ldots,n_{d-1} | n_d)$. We will now briefly recall their definition.

An architecture $(n_0,\ldots,n_{d-1} | n_d)$ determines a parameterized family $\mathcal{F}$ as follows.  
The parameter space is $\Omega = \mathbb{R}^D$, where $D \coloneqq \sum_{\ell=1}^d n_{\ell}(n_{\ell-1} +1)$. The map $\mathcal{F}:\Omega \times \mathbb{R}^{n_0} \to \mathbb{R}^{n_d}$ is defined as follows. Let $\sigma: \mathbb{R}^n \rightarrow \mathbb{R}^n$ denote the function that applies the activation function $\mbox{ReLU}(x):= \max\{0,x\}$ component-wise. A parameter $\theta \in \Omega$ determines a collection (indexed by $1 \leq \ell \leq d$) of $n_{\ell} \times n_{\ell-1}$ weight matrices and $n_{\ell} \times 1$ bias vectors $b^\ell$; we think of the vector $\theta$ as  $\theta := (W^1,b^1, \ldots, W^d, b^d) \in \Omega$,  obtained by unrolling the weight and bias matrices in the standard way.  Then we define the associated ReLU neural network function $F_{\theta}$ as a composition of layer maps: 
\begin{equation} \label{eqn:ReLUFunction} F_\theta: \xymatrix{\mathbb{R}^{n_0} \ar[r]^-{F^1} & \mathbb{R}^{n_1} \ar[r]^-{F^2} & \ldots \ar[r]^-{F^d} & \mathbb{R}^{n_d}},\end{equation} with layer maps given by: \begin{equation} \label{eqn:layermap}
F^\ell(x) := \left\{\begin{array}{cl}\sigma(W^\ell x + b^\ell) & \mbox{for $1 \leq \ell < d$}\\
W^\ell x + b^\ell & \mbox{for $\ell = d$}.\end{array}\right.\end{equation}
(That is, each layer map except the last is the composition of an affine-linear map followed by $\sigma$, which is component-wise application of $\mbox{ReLU}$.) Then the map $\mathcal{F}:\Omega \times \mathbb{R}^{n_0} \to \mathbb{R}^{n_d}$ is defined by setting $\mathcal{F}(\theta,x) = F_\theta(x)$.

Given a composition of layer maps as in \eqref{eqn:ReLUFunction}, we will (following \cite[Definition 4]{Masden}) use the notation 

\begin{align} \label{eq:subscriptnotation}
F_{(\ell)} & \coloneqq F^{\ell} \circ \ldots \circ F^1.
\end{align}
It is well-known (cf. \cite{arora2018,HaninRolnick,GrigsbyLindsey, GLMW}) that the class of ReLU neural network functions $\mathcal{F}:\Omega \times \mathbb{R}^{n_0} \to \mathbb{R}^n$ is finitely piecewise-polynomial. 

\subsection{Piecewise polynomial parameterized function families} \label{sec:PPFF}

In this subsection, we focus on parameterized families $\mathcal{F}$ that are not only continuous but also finitely piecewise polynomial (as is the case for ReLU neural network families).
A central theme of the results in this paper is proving relations (e.g. inequalities) between the ranks of various matrices whose entries are polynomial functions of the parameter $\theta$ and/or the input point $x \in \mathbb{R}^{n_0}$.  Typically, the rank (whatever definition one is using of rank) of such a matrix will be constant for most values of $\theta$ and/or $x$, but there will certain values of $\theta$ and $x$ where the rank of the matrix will drop below this value (when, for example, specific values of $x$ and $\theta$ cause a linear dependency to develop between the rows or columns of the matrix).  Thus, 
in order to prove results about the behavior of the parameterized function family in an \emph{open neighborhood} of a point $(x,\theta) \in \mathbb{R}^{n_0} \times \Omega$, we may need to know not only the rank a matrix assumes at that specific $(x,\theta)$ (which could be one of the ``bad'' points where the rank drops), but the ``typical'' rank of the matrix on that neighborhood.  To that end, we must investigate the polynomial expressions that describe the local behavior of the function family, and the matrices which have these formal polynomials as their entries, rather than their real-valued instances at specific points $(x,\theta)$. This becomes particularly important in Section \ref{sec:Rkpoly}.

 In the case of finitely piecewise polynomial families,  
since continuity implies that the pieces of $\Omega \times \mathbb{R}^{n_0}$ (which by definition contain a nonempty open subset) on which $\mathcal{F}$ is a polynomial are defined by polynomial equations; in fact, it is straightforward to show that the pieces are \emph{basic closed semialgebraic sets} (by definition, a basic closed semialgebraic set is a set $S \subset \mathbb{R}^n$ given by finitely many simultaneous polynomial inequalities, $S = \{x \in \mathbb{R}^n \mid p_1(x) \geq 0, p_2(x) \geq 0,\ldots, p_r(x) \geq 0\}$ where $p_i(x) \in \mathbb{R}[x]$, $1 \leq i \leq r$, are polynomials.)  Note that i) both $\mathbb{R}^{n_0}$ and $\Omega$ are basic closed algebraic subsets of $\Omega\times \mathbb{R}^{n_0}$, the domain of $\mathcal{F}$, and ii) a finite intersection  of basic closed semialgebraic sets is basic closed semialgebraic.  Consequently, both every $F_\theta:\mathbb{R}^{n_0} \to \mathbb{R}^{n_d}$ and every $E_z:\Omega \to \mathbb{R}^{n_d}$ are  finitely piecewise-polynomial with pieces that are basic closed semialgebraic sets. 

 Since each piece of $\Omega \times \mathbb{R}^{n_0}$ contains an open set (by the definition of ``piece''), the Identity Theorem for polynomials in several variables guarantees that there is a unique \emph{formal polynomial} -- an element of the ring $\mathbb{R}[\theta_1,\ldots,\theta_D,x_1,\ldots,x_{n_0}]\,$
 the polynomial ring over $D+n_0$ formal variables -- that represents the restriction of $\mathcal{F}$ to the piece. 

Given a formal polynomial $r \in \mathbb{R}[\theta_1,\ldots,\theta_D,x_1,\ldots,x_{n_0}]$, we will sometimes want to ``plug in'' real values for either the variables $\theta_1,\ldots,\theta_D$ or the variables $x_1,\ldots,x_{n_0}$; we now define notation for this. 
For
 $$r \in \mathbb{R}[\theta_1,\ldots,\theta_D,x_1,\ldots,x_{n_0}],$$ $z \in \mathbb{R}^{n_0}$ and $\tilde{\theta} \in \Omega$, we will denote by 
\begin{equation} \label{eq:substitutex} 
r/(x \to z)\end{equation} the element of the polynomial ring 
$$R \coloneqq \mathbb{R}[\theta_1,\ldots,\theta_D]$$ obtained from $r$ by replacing the formal variables $x_1,\ldots,x_n$ with the real numbers that are the corresponding coordinates of $z$;  we will also denote by 
\begin{equation} \label{eq:substitutetheta} r/(\theta \to \tilde{\theta})
\end{equation} the element of the polynomial ring $\mathbb{R}[x_1,\ldots,x_{n_0}]$ obtained by replacing the formal variables $\theta_1,\ldots, \theta_D$ with the real numbers that are the corresponding coordinates of $\tilde{\theta}$.

\begin{definition}[algebraic representation, algebraic evaluation map] \label{def:algebraicversions}
 Let $\mathcal{F}$ be a finitely piecewise polynomial parameterized function family on $\mathbb{R}^{n_0}$ with parameter space $\Omega$.  
 \begin{itemize}
     \item The \emph{algebraic representation of $\mathcal{F}$} is the map 
\begin{multline*}
\textrm{P}: \{(\theta,x) \in \Omega \times \mathbb{R}^{n_0} | (\theta,x) \textrm{ is a smooth point for } \mathcal{F}\} \\
\to \mathbb{R}[\theta_1,\ldots,\theta_D,x_1,\ldots,x_{n_0}]
\end{multline*}
defined by setting $P(\theta,x)$ to be the unique element of $\mathbb{R}[\theta_1,\ldots,\theta_D,x_1,\ldots,x_{n_0}]$ that represents $\mathcal{F}$ on an open neighborhood of $(\theta,x)$ in $\Omega \times \mathbb{R}^{n_0}$. 
\item For any fixed $z \in \mathbb{R}^{n_0}$, the \emph{algebraic evaluation map} for $z$ is the map  
\[E_z^{R}: \{\theta \in \Omega: (\theta,z) \textrm{ is a smooth point for } \mathcal{F} \} \to R\]
given by 
$$E_z^R(\theta) = P(\theta,x)/(x \to z),$$
i.e. it assigns to $\theta$ the formal polynomial (in $\theta_1$, \ldots, $\theta_D$) that represents $\mathcal{F}(-,z)$ near $\theta_0$. 

\item For any finite set $Z = \{z_1,\ldots,z_m\} \subset \mathbb{R}^{n_0}$, the \emph{algebraic evaluation map} for $Z$ is the map  
\[E_Z^{R}: \{\theta \in \Omega: (\theta,z_i) \textrm{ is a smooth point for } \mathcal{F} \textrm{ for all } z_i \in Z\} \to R^m\]
given by 
$$E_Z^R(\theta) = (E^R_{z_1}(\theta),\ldots,E^R_{z_m}(\theta)).$$
 \end{itemize}
\end{definition}

\begin{remark}
    The reason for only defining $P$, (and hence also $E_z^R$) at parameters $\theta$ for which $(\theta,z)$ is a smooth point for $\mathcal{F}$ is that the smoothness implies there is a unique element of $\mathbb{R}[\theta_1,\ldots,\theta_D,x_1,\ldots,x_{n_0}]$ that represents $\mathcal{F}$ near such a point, and so the map is well-defined. 
    
    The reason for defining the map $\theta \mapsto E_Z^R(\theta)$ by starting with a formal polynomial  $P \in \mathbb{R}[\theta_1,\ldots,\theta_D,x_1,\ldots,x_{n_0}]$ and then  substituting in real numbers for formal variables in $P$ is that we want to keep track of how each variable appears in the polynomial, i.e. we need $E_Z^R(\theta)$ to be well-defined.   The issue is that a polynomial in $D+n_0$ variables is not uniquely defined by specifying it's values on a proper (e.g. $D$-dimensional or $n_0$-dimensional) linear subspace. 
\end{remark}

\begin{remark} \label{rem:JacDerivOfAlgebraic}
There is a natural way to compute the ``Jacobian derivative'' of $E_Z^R$ at a point $\theta$ in its domain -- we just apply the usual differentiation rules to the formal polynomials that are the coordinates of $E_Z^R(\theta)$.  We will use $\mathbf{J}E_Z^R(\theta)$ to denote this $|Z| \times D$ matrix over $R$. See Example \ref{ex:algebraicJacobian} below. 
\end{remark}

\begin{example} \label{ex:algebraicJacobian} 
    The parameterized family for the architecture $(1,2|1)$ is the map 
    $\mathcal{F}: \mathbb{R}^7 \times \mathbb{R}^1 \to \mathbb{R}^1$ given by 
    \[((\theta_1,\ldots,\theta_7), x) \mapsto \theta_5 \sigma(\theta_1x + \theta_2) + \theta_6 \sigma(\theta_3x+\theta_4) + \theta_7.
    \]
    
\noindent  Consider the parameter
    $\tilde{\theta} = (1,0,1, -1, 1,1,0),$ and points $z_0 = -1, z_1 =\tfrac{1}{2}, z_2 =\tfrac{3}{2}$. 
    We have
    \begin{align*}     
    P(\tilde{\theta},z_0) & = \theta_7,  \\ 
    P(\tilde{\theta},z_1) & = \theta_5(\theta_1x_1+\theta_2) + \theta_7, \\  
    P(\tilde{\theta},z_2) & = \theta_5(\theta_1x_1 + \theta_2) + \theta_6 (\theta_3x_1+\theta_4) + \theta_7.
    \end{align*}   
    because 
$\sigma(\theta_1x+\theta_2) = \sigma(\theta_3x+\theta_4) = 0$ for all $(\tilde{\theta},x)$ sufficiently close to $(\theta,z_0)$ and 
$\sigma(\theta_3x+\theta_4) = 0$ for all $(\tilde{\theta},x)$ sufficiently close to $(\theta,z_1)$.
Then 
\begin{align*}
    E^R_{z_0}(\tilde{\theta}) &= \theta_7/(x \to z_0) = \theta_7, \\
    \begin{split} E^R_{z_1}(\tilde{\theta}) &= \left( \theta_5(\theta_1x_1+\theta_2) + \theta_7 \right)/(x \to z_1)   \\ 
     & = \theta_5 \left(\frac{\theta_1}{2}+\theta_2\right) + \theta_7, \end{split} \\
 \begin{split} E^R_{z_2}(\tilde{\theta})  & = \left( \theta_5(\theta_1x + \theta_2) + \theta_6 (\theta_3x+\theta_4) + \theta_7 \right)/(x \to z_2) \\
 & =  
 \theta_5 \left(\frac{3\theta_1}{2} + \theta_2 \right) + \theta_6 
 \left(\frac{3\theta_3}{2}+\theta_4\right) + \theta_7.
 \end{split}
\end{align*}
Hence
\[ \mathbf{J}E^R_{\{z_0,z_1,z_2\}}(\tilde{\theta})
= \begin{bmatrix}
    \mathbf{J}E_{z_0}^R(\tilde{\theta}) \\ \mathbf{J}E_{z_1}^R(\tilde{\theta}) \\  \mathbf{J}E_{z_2}^R(\tilde{\theta})
\end{bmatrix} 
 = \begin{bmatrix}
    0 & 0 & 0 & 0 & 0 & 0 & 1 \\
    \frac{\theta_5}{2} & \theta_5 & 0 & 0 & \frac{\theta_1}{2} + \theta_2 & 0 & 1 \\
    \frac{3\theta_5}{2} & \theta_5 & \frac{3\theta_6}{2} & \theta_6 & \frac{3\theta_1}{2} + \theta_2 & \frac{3\theta_3}{2} + \theta_4 & 1
\end{bmatrix} \]
\end{example}

\subsection{ReLU networks: geometric and combinatorial background} \label{ss:GeometricCombinatorialBackground}

It is well-known that if $\mathcal{F}$ is a class of (standardly parameterized) ReLU network functions of fixed architecture $(n_0, \ldots, n_{d-1}|n_d)$, then for every parameter $\theta \in \Omega$, the corresponding function $F_\theta: \mathbb{R}^{n_0} \rightarrow \mathbb{R}^{n_d}$ is piecewise-linear with finitely many pieces, cf. \cite{arora2018, TropGeometry}. The locus of non-linearity for $F_\theta$ is contained in its so-called {\em fold set} or {\em bent hyperplane arrangement}, which is simply the union of the zero sets of the pre-activated hidden neurons in the network, cf. \cite{HaninRolnick,RolnickKording,PhuongLampert}. 

In \cite{GrigsbyLindsey}, we defined the notion of a {\em generic, transversal} 
parameter $\theta \in \Omega$ in order to formalize a condition ensuring that the bent hyperplane arrangement was codimension $1$ (as expected) in the domain. The latter notion was extended by Masden in \cite{Masden} to give a much improved version notion of transversality called {\em supertransversality}. 
We briefly remind readers of the relevant notions here, and refer them to \cite{GrigsbyLindsey, Masden} for more details.

 Let $\theta \in \Omega$ be a fixed parameter for $\mathcal{F}$ the parameterized class of feedforward ReLU networks of architecture $(n_0, \ldots, n_{d-1}|n_d)$, and let $F_\theta: \mathbb{R}^{n_0} \rightarrow \mathbb{R}^{n_d}$ the corresponding function from the input to the output space. Letting $z^\ell_i = \pi_i (W^\ell x + b^\ell)$ denote the $i$th component of the pre-activation output of the $\ell$th layer map $F^\ell:\mathbb{R}^{n_{\ell-1}} \to \mathbb{R}^{n_\ell}$ of $F_\theta$, we denote its zero set by $H^\ell_i := (z^\ell_i)^{-1}\{0\} \subseteq \mathbb{R}^{n_{\ell - 1}}.$ Note that for almost all parameters for $F^\ell$, $H^{\ell}_i$ is an affine hyperplane. Accordingly, we associate to each layer map $F^\ell$ the set \[\mathcal{A}^{\ell} = \{H^\ell_1, \ldots, H^\ell_{n_\ell}\} \subseteq \mathbb{R}^{n_{\ell-1}},\] which for almost all parameters is a hyperplane arrangement (cf.~\cite{Stanley}). 
 $\mathcal{A}^\ell$ is said to be \emph{generic} if for all subsets \[\{H^\ell_{i_1} , \ldots , H^\ell_{i_k}\} \subseteq \mathcal{A}^\ell,\] it is the case that $H^\ell_{i_1} \cap \ldots \cap H^\ell_{i_k}$ is an affine-linear subspace of $\mathbb{R}^{n_{\ell-1}}$ of codimension $k$ (dimension $n_{\ell-1} - k$), where a negative-dimensional intersection is understood to be empty. 
 A layer map $F^\ell$ is said to be \emph{generic} if $\mathcal{A}^\ell$ is generic.  A parameter $\theta$ or the corresponding network map $F_\theta$ is said to be generic if all of its layer maps are generic.
 It is well-established in the hyperplane arrangement literature (cf.~\cite{Stanley}) that generic arrangements are full measure. It follows \cite{GrigsbyLindsey} that generic network maps are full measure in parameter space.

The bent hyperplanes are the pull-backs of these zero sets to the domain:
\[\hat{\mathcal{A}}^\ell = \{\hat{H}^\ell_i\}_{i=1}^{n_\ell} := \left\{F_{(\ell -1)}^{-1}(H_i^\ell)\right\}_{i=1}^{n_\ell} \subseteq \mathbb{R}^{n_0}.\]
(Recall that the notation $F_{(\ell)}$ is defined in \eqref{eq:subscriptnotation}.)

Recalling that a {\em polyhedral set} is an intersection of finitely many closed half spaces, a hyperplane arrangement in $\mathbb{R}^{n_{\ell-1}}$ induces a {\em polyhedral decomposition} of $\mathbb{R}^{n_{\ell -1}}$ into finitely many polyhedral sets. The boundary face structure on these polyhedral sets gives the decomposition the structure of a polyhedral complex. By pulling back these polyhedral complexes to the domain, $\mathbb{R}^{n_0}$, and taking intersections, we inductively obtain the \emph{canonical polyhedral complex} $\mathcal{C}(F_\theta)$, defined in \cite{GrigsbyLindsey} (see also \cite{Masden}), as follows. 
 
 For $\ell \in \{1,\ldots,d\}$, denote by $R^{\ell}$ the polyhedral complex on $\mathbb{R}^{n_{\ell-1}}$ induced by the hyperplane arrangement associated to the $\ell^{\textrm{th}}$ layer map, $F^{\ell}$.  Inductively define polyhedral complexes $\mathcal{C}\left(F_{(1)}\right),\ldots, \mathcal{C}\left(F_{(d)}\right)$   on $\mathbb{R}^{n_0}$ as follows: Set $\mathcal{C}\left(F_{(1)} = F^{1}\right):= R^{1}$ and for $i = 2,\ldots,m$, set
\begin{equation*}
\mathcal{C}(F_{(\ell)})
:= \left \{S \cap F_{(\ell-1)}^{-1}(Y) \mid S \in \mathcal{C}\left(F_{(\ell-1)}\right), Y \in R^{\ell} \right \}.
 \end{equation*} 
Set $\mathcal{C}(F_\theta) := \mathcal{C}(F_\theta = F_{(d)})$.  See \cite{GrigsbyLindsey} and \cite{Masden} for more details.

Recall (cf.~\cite{GuilleminPollack} and Section 4 of \cite{GrigsbyLindsey}) the following important classical notion, which ensures that -- in a sense which can be made precise -- the intersection of two submanifolds of a manifold is itself a manifold whose codimension is the sum of the codimensions of the intersecting manifolds:

\begin{definition} [\cite{GuilleminPollack}] \label{defn:maptransverse} Let $X$ be a smooth manifold with or without boundary, $Y$ and $Z$ smooth manifolds without boundary, $Z$ a smoothly embedded submanifold of $Y$, and $f:X \to Y$ a smooth map. We say that $f$ is {\em transverse} to $Z$ and write $f \pitchfork Z$ if 
\begin{equation} \label{eq:maptransverse}
df_p(T_pX) + T_{f(p)}Z = T_{f(p)}Y
\end{equation} 
for all $p \in f^{-1}(Z)$.
\end{definition}

One important contribution of \cite{GrigsbyLindsey} and \cite{Masden} is the development of a consistent and robust theory of transversality for polyhedral complexes which has good behavior with respect to the combinatorics of ReLU neural network decision regions and boundaries. We refer the reader to the original works \cite{GrigsbyLindsey, Masden} for details, as well as \cite[App.A-B]{GLR} for a more concise exposition. We won't repeat the precise definition of supertransversality here, as it is technical. Instead, we will remind readers of the properties of generic, supertransversal networks that will be most important to us in the current work.

Informally, we can view supertransversality as the right generalization of the genericity condition for hyperplane arrangements to bent hyperplane arrangements. As noted above, a hyperplane arrangement is {\em generic} if every $k$--fold intersection of hyperplanes in the arrangement is an affine linear subspace of dimension $n-k$.  Analogously, it follows from the definitions and results in \cite{Masden} that every $k$--fold intersection of bent hyperplanes associated to a generic, supertransversal network is a (possibly empty) polyhedral complex of dimension $n-k$. 

The crucial result for us, proved in \cite{Masden} (see also Theorem 3 of \cite{GrigsbyLindsey}), is the following:

\begin{proposition}[Lemma 12 of \cite{Masden}] \label{p:genstfullmeasure} For any neural network architecture, the set of parameters associated to generic, supertransversal neural network functions is full measure in parameter space, $\mathbb{R}^D$.
\end{proposition}

The cells of the canonical polyhedral complex, $\mathcal{C}(F_\theta)$, associated to a parameter $\theta \in \Omega$ naturally comes with a ternary labeling, as defined in \cite{Masden}: 
\begin{definition} \label{def:ternaryactivationpattern}
A {\em ternary activation pattern} (aka {\em ternary neural code} or {\em ternary sign sequence}) for a network architecture $(n_0, \ldots, n_{d-1}|n_d)$ with $N$ neurons is a ternary tuple $s \in \{-1,0,+1\}^N$. The \emph{ternary labeling} of a point $x \in \mathbb{R}^{n_0}$ is the sequence of ternary tuples $$s_x := \left(s_x^{1}, \ldots, s_x^{d}\right) \in \{-1,0,+1\}^{n_1 +\ldots + n_d}$$ indicating the sign of the pre-activation output of each neuron of $F_\theta$ at $x$. \end{definition} 

Explicitly, let $x \in \mathbb{R}^{n_0}$ any input vector, and suppose that the pre-activation output $z_{(\ell),i}(x)$ of the $i$th neuron in the $\ell$th layer at $x$ is as in Equation \ref{eqn:preactneuron}. Then  $s_x^{\ell} = \left(s^{\ell}_1, \,\, \ldots \,\,, s^{\ell}_{n_\ell}\right)$ are defined by $s^{\ell}_{x,i} = \mbox{sgn}(z_{(\ell,i)}(x))$ (using the convention $\textrm{sgn}(0) = 0$). 

Moreover, for all parameters $\theta$ it follows immediately from the definitions that the ternary labeling is constant on the interior of each cell of $\mathcal{C}(F_\theta)$, inducing a ternary labeling $s_C$ on each cell $C$ of $\mathcal{C}(F_\theta)$.

If $\theta \in \Omega$ is generic and  supertransversal, the ternary labeling, $s_C$, gives us strong information about a cell $C$. In particular:
\begin{enumerate}[(1)]
    \item If $s_C$ has $k$ $0$s, then the dimension of its associated cell $C$ is $n_0 - k$. 
    \item If the ternary labeling $s_C$ for $C$ has $k$ $0$'s, then there are $2^k$ non-empty higher-dimensional cells in $\mathcal{C}(F_\theta)$ for which $C$ is a face. Moreover, their ternary labelings are obtained from $s_C$ by replacing the $0$s in $s_C$ with $\{\pm 1\}$s.
\end{enumerate}

\begin{definition} \label{def:activationregion}
The \emph{activation region} of $F_\theta$ corresponding to a ternary activation pattern $s$ is a maximal connected component of the set of input vectors $x \in \mathbb{R}^{n_0}$ for which the ternary labeling $s_x$ equals $s$.
\end{definition}

\begin{definition} \label{Def:pmactivationregion}
A \emph{$\pm$-activation pattern} is a ternary activation pattern in which every coordinate is nonzero, and a \emph{$\pm$-activation region} is an activation region associated to a $\pm$-activation pattern. 
\end{definition}
Note that any $\pm$-activation region is an open set.

The important implication for us in the current work is the following. See Definition \ref{def:stablerealrankset} to recall the definition of {\em algebraically stable set} and {\em stable real rank set} for a finite batch of points $Z$ in a finitely piecewise polynomial family:

\begin{proposition} \label{p:stabRrksetbig}
Let $\mathcal{F}$ be a parameterized family of ReLU neural networks of fixed architecture $(n_0, \ldots, n_{d-1}|n_d)$, and let $Z = \{z_1, \ldots, z_m\} \subseteq \mathbb{R}^{n_0}$ be any finite batch of points in the domain. Almost all $\theta \in \Omega$ are in the stable real rank set of $Z$.
\end{proposition}

\begin{proof}
Let $\mathcal{G} \subset \Omega$ denote the set of $\theta \in \Omega$ that are generic and supertransversal. Since we know from \cite[Lem. 12]{Masden} (Prop. \ref{p:genstfullmeasure}) that $\mathcal{G}$ is full measure in $\Omega$, and we know from Lemma \ref{l:RRstabfullinalgstab} that the stable real rank set is full measure in the algebraically stable set of any finite batch $Z$, it suffices to show that we can describe the algebraically stable real rank set of $Z$ as $\mathcal{G} \setminus \mathcal{S}$ for some Lebesgue measure zero set $\mathcal{S} \subset \Omega$.

By \cite[Lem. 16, Lem. 19, Thm. 20]{Masden}, we know that the set of generic, supertransversal $\theta \in \mathcal{G} \subseteq \Omega$ for which $Z$ is parametrically smooth is those $\theta \in \mathcal{G}$ with respect to which the ternary sign sequence $s_{z_i}$ has no $0$s, for each $z_i \in Z$. But the condition that $s_{z_i}$ has no $0$'s for $\theta \in \Omega$ is precisely the condition that when we evaluate the polynomial associated to each pre-activated neuron at $(\theta, z_i)$, the result is nonzero. Moreover, as long as none of the biases is 0 (a polynomial condition in the parameters), each of the polynomials associated to the pre-activated neurons in the network is a nonzero polynomial. 

We conclude that any finite batch $Z \subseteq \mathbb{R}^{n_0}$ is parametrically smooth in $\mathcal{G} \setminus \mathcal{S}$ for $\mathcal{S}$ an algebraic set. The result follows.
\end{proof}

\begin{remark}
As noted in the proof above, the algebraically stable set for $Z$ also has full Lebesgue measure in $\Omega$.
\end{remark}

\subsection{Encoding algebraic representations via computational graphs}

\label{sec:CompGraph}
In this section, we remind the reader of the relationship between the unique formal polynomial(s) described in Definition \ref{def:algebraicversions}
and active/inactive neurons of the computational graph of the architecture. 

Explicitly, the formal polynomial associated to a $\pm$ activation region is a sum of formal monomials in the input variables $x_1, \ldots, x_{n_0}$, weight parameters, and bias parameters, where each monomial is specified uniquely by a path in the augmented computational graph for the architecture. See \cite[Lem. 8]{HaninRolnick} and \cite[App. D]{GLR}. We give more details below.

\begin{definition} \label{defn:weightedcomputationalgraph} The {\em augmented} computational graph $\tilde{G}$ for the feedforward ReLU network architecture $(n_0, \ldots | n_d)$ is the graded oriented graph: 
\begin{itemize}
    \item with $n_\ell$ ordinary vertices and $1$ distinguished vertex of grading $\ell$ for $\ell = 0, \ldots, d-1$, and $n_d$ ordinary vertices of grading $d$,
    \item for every $\ell = 0, \ldots, d-1$, every vertex of grading $\ell$ is connected by a single oriented edge to every ordinary vertex of grading $\ell+1$, oriented toward the vertex of grading $\ell+1$.
\end{itemize} 
\end{definition}

In particular, note that one obtains the augmented computational graph for an architecture from the standard computational graph for the architecture by adding an extra marked vertex for each non-output layer, whose purpose is to record the bias parameter in each affine-linear map. Accordingly, one obtains a labeling of the edges of the augmented computational graph by formal weight and bias variables as follows:
\begin{itemize}
    \item the edge from the distinguished vertex of layer $\ell$ to the $k$th ordinary vertex of layer $\ell+1$ is labeled with $b^{\ell+1}_k$, the $k$th  component of the bias vector for $F^{\ell+1}$,
    \item the edge from the $i$th ordinary vertex of layer $\ell-1$ to the $j$th ordinary vertex of layer $\ell$ is labeled with $W^{\ell}_{ji}$.
\end{itemize}

Associated to every oriented path $\gamma$ is a corresponding formal monomial, $m(\gamma)$, in the weight and bias variables obtained by taking the product of the parameters on the edges traversed along $\gamma$. See Figure \ref{fig:AugmentedCompGraph}.

Recall the notation $F_{(\ell)} := F^\ell \circ \ldots \circ F^1$ from \eqref{eq:subscriptnotation}.
We refer to the components of $F_{(\ell)}$ as the {\em neurons} in the $\ell$th layer. The {\em pre-activation} map $y_{(\ell),i}: \mathbb{R}^{n_0} \rightarrow \mathbb{R}$ associated to the $i$th neuron in the $\ell$th layer is given by: 
\begin{equation} \label{eqn:preactneuron}
 y_{(\ell),i}(x) = \pi_i\left(W^\ell(F_{(\ell-1)}(x)) + b^\ell\right),
 \end{equation} where $\pi_i: \mathbb{R}^{n_\ell} \rightarrow \mathbb{R}$ denotes the projection onto the $i$th component.

The $i$th neuron in layer $\ell$ is said to be {\em on} or {\em active} (resp, {\em off} or {\em inactive}) at $x$ if $y_{(\ell),i}(x) > 0$ (resp., if $y_{(\ell),i}(x) \leq 0$). Let $A_x^{\ell} \subseteq \{1, \ldots, n_\ell\}$ denote the subset of neurons in layer $\ell$ that are active at $x$.

\begin{definition}
\label{defn:opencompletepath}
Let $\tilde{G}$ be the augmented computational graph for the ReLU network architecture $(n_0, \ldots, n_d)$, with $n_d=1$. A path $\gamma$ is said to be {\em complete} if it ends at a vertex in the output layer and begins at either a vertex of the input layer or at one of the distinguished bias vertices in a non-input layer. We will denote by  $\Gamma$ the set of complete paths of $\tilde{G}$. For $\theta \in \Omega$ and $x \in \mathbb{R}^{n_0}$, a path $\gamma \in \Gamma$ is said to be {\em open} at $x$ if all neurons along $\gamma$ are active. For $\theta \in \Omega$ we will denote by:  
\begin{itemize} 
    \item  $\Gamma_{x}^{\theta} \subseteq \Gamma$ the subset of complete paths that are open at $x$ for the parameter $\theta$. 
    \item $\Gamma_{x,i}^{\theta} \subseteq \Gamma_x^{\theta}$ the subset of $\Gamma_x^{\theta}$ beginning at input node $i$, and 
    \item $\Gamma_{x,*}^{\theta}\subseteq \Gamma_x^{\theta}$ the subset of $\Gamma_x^{\theta}$ beginning at one of the distinguished bias vertices.
\end{itemize}  
\end{definition}

Note that \[|\Gamma| = \sum_{j=0}^{d-1} \left(\prod_{\ell=j}^d n_\ell\right),\] and \[\left|\Gamma_x^\theta\right| = \sum_{j=0}^{d-1}\left(\prod_{\ell=j}^d |A^\ell_x|\right).\]

A version of the following lemma is well-known to the experts (cf. Lemma 8 of \cite{HaninRolnick}). We state it here in the language of Section \ref{sec:PPFF} and Definition \ref{def:algebraicversions}.

\begin{lemma} \label{lem:poly}
Let $\mathcal{F}$ be a parameterized family of ReLU neural networks of architecture $(n_0, \ldots, n_{d-1}|n_d)$, $\theta \in \Omega$ an ordinary parameter, and let $x  \in \mathbb{R}^{n_0}$ be a parametrically smooth point in the domain. Then abusing notation to let $\theta_1, \ldots, \theta_D$ denote the standard coordinate monomials of $\Omega = \mathbb{R}^D$ and $x_1, \ldots, x_{n_0}$ the standard coordinate monomials of the domain, the {\em algebraic representation} of $\mathcal{F}$ in a neighborhood of $(\theta,x)$ (Definition \ref{def:algebraicversions}) is computed as follows from the weighted computational graph:
\[P(\theta,x) =  \sum_{\gamma \in \Gamma_{x,*}^{\theta}}  m(\gamma) + \sum_{i=1}^{n_0} x_i\sum_{\gamma \in \Gamma_{x,i}^{\theta}} m(\gamma).\]
\end{lemma}

\noindent (The notation $m(\gamma)$ was defined in the paragraph preceding that containing \eqref{eqn:preactneuron}. See also Figure \ref{fig:AugmentedCompGraph} for an example.) 

\begin{remark} The well-known version of the lemma above (cf. \cite[Lem. 8]{HaninRolnick}, \cite[]{GLR}) gives an algorithm for computing $\mathcal{F}(\theta,x)$ as an element of $\mathbb{R}$ for any $(\theta, x) \in \Omega \times \mathbb{R}^{n_0}$. In the version above, we note that we can analogously extract the algebraic representation $P(\theta,x)$ at smooth points $(\theta,x)$. It is straightforward -- but notationally fiddly -- to give a canonical extension of the assignment of formal polynomials to {\em non-smooth} points $(\theta,x)$ in the case where $\theta$ is generic and supertransversal. We briefly sketch the construction here, since we won't need it in the current work, but it may be of use to others. 

It follows from \cite{GrigsbyLindsey, Masden} that if $\theta$ is generic and supertransversal, and $(\theta,x)$ is a non-smooth point, then $s_\theta(x)$ has at least one entry that is $0$. In fact, one of the main results of \cite{Masden} is that if $k$ is the number of $0$s in $s_\theta(x)$, then $x$ is contained in the interior of a codimension $k$ cell of the canonical polyhedral complex of $F_\theta$. Moreover, since $\theta$ is generic and supertransversal, it follows from the main results of \cite{Masden} that $x$ is in the boundary of the $\pm$ activation region for which all of the $0$s in $s_\theta(x)$ have been replaced with $-1$s. Since a neuron is defined to be {\em inactive} if its pre-activation sign is {\em non-positive}, and the function $F_\theta$ is continuous, it is then natural to define $P(\theta,x)$ in this case to be equal to $P(\theta,x^*)$ where $x^*$ is a point in the $\pm$ activation region described above. Verifying that this algebraic representation is canonical (well-defined and unique in a neighborhood of $(\theta,x) \in \Omega \times V$, where $V \subseteq \mathbb{R}^{n_0}$ is the affine hull of the unique cell whose interior contains $x$) should also be routine.
\end{remark}
\begin{figure}
\begin{center}
\includegraphics[width=2in]{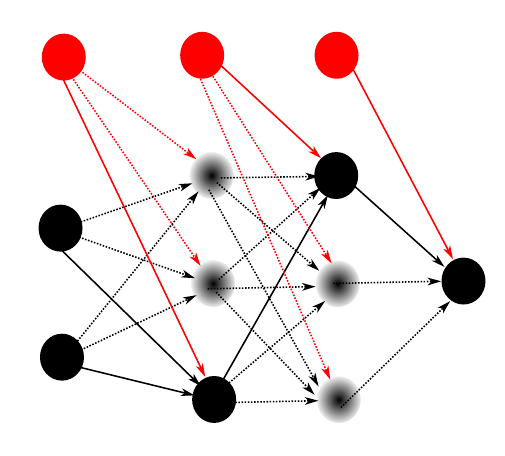}    
\end{center}
  \caption{An augmented computational graph for architecture (2,3,3,1). The ordinary vertices are black, and the distinguished vertices are red. Black edges are labeled with weights, and red edges are labeled with biases. A complete path is one that ends at an output vertex and begins either at an input vertex or at one of the distinguished vertices. In the diagram above, we have blurred out vertices corresponding to inactive neurons associated to an input vector $x$ for some parameter $\theta$. The open paths are then the ones in the diagram above with solid (non-dashed) edges. The reader can check that there are three open, complete paths $\gamma, \gamma', \gamma'' \in \Gamma^\theta_{x,*}$,  whose  monomials are $m(\gamma) = b_3^1W^2_{13}W^3_{11}$, $m(\gamma') = b_1^2W^3_{11}$, and $m(\gamma'') = b_1^3$. There is a unique open, complete path $\gamma_1 \in \Gamma^\theta_{x,1}$, with monomial $m(\gamma_1) = W^1_{31}W^2_{13}W^3_{11}$ and a unique open, complete path $\gamma_2 \in \Gamma^\theta_{x,2}$, with monomial $m(\gamma_2) = W^1_{32}W^2_{13}W^3_{11}$.}
  \label{fig:AugmentedCompGraph}
\end{figure}

\subsection{Rank of matrices over polynomial rings} \label{sec:Rkpoly}

Our investigation of the relationship between 
 the persistent pseudodimension and the local functional dimension of a piecewise polynomial parameterized family will require several different notions of rank for matrices over polynomial rings.

Since we will also encounter matrices over the field $\mathbb{R}$, we will be diligent about staying consistent in our notation. If $M$ is a matrix over $\mathbb{R}$, we will refer to its rank as $\mbox{rank}(M)$. Note that this matches the notation from \cite{GLMW}, see also Definition  \ref{def:batchfundim}.

As before, let $R = \mathbb{R}[\theta_1, \ldots, \theta_D]$ be the real polynomial ring over $D$ formal variables. We have in mind that $\theta_1, \ldots, \theta_D$ represent parameters of a feedforward ReLU network of fixed architecture. In this section, we will appeal to general results about linear algebra over commutative rings with $1$, but note that since $R = \mathbb{R}[u_1, \ldots, u_D]$ is an integral domain (that is, if $r,s \in R$ satisfies $rs = 0$, then $r=0$ or $s=0$), some of the more pathological situations (cf. \cite{McD}) do not occur.

Recall that an ideal $I \subseteq R$ is any non-empty subset of $R$ satisfying  $0 \in I$; $i_1 + i_2 \in I$ for all $i_1, i_2 \in I$; and $ri \in I$ for all $i \in I, r \in R$.
Note that an ideal $I \subseteq R$ is, in particular, a two-sided module over $R$ for which the left and right actions of $R$ by scalar multiplication agree.

The only other modules we will encounter in the present work are free modules, $R^n$, for $n \in \mathbb{N}$, with the standard scalar multiplication. Viewing $v \in R^n$ as a column ``vector", an $m \times n$ matrix $A$ over $R$ naturally gives rise to a module map $R^n \rightarrow R^m$ via left matrix multiplication by $A$: $v \mapsto Av.$ We will use the standard notation $A^T$ to denote the {\em transpose} of $A$, which is the $n \times m$ matrix whose $i$th row is the $i$th column of $A$.

\begin{definition}  \ 
\begin{itemize}  
\item The columns (resp. rows) of $A$ are said to be \emph{linearly independent over $R$} if the unique $v \in R^n$ (resp., the unique $w \in R^m$) for which $Av = 0$ (resp., for which $A^Tw = 0$) is $v = 0$ (resp., $w = 0$). 
\item  The columns (resp. rows) of $A$ are said to be \emph{linearly independent over $\mathbb{R}$} if the unique $v \in \mathbb{R}^n$ (resp., the unique $w \in \mathbb{R}^m$) for which $Av = 0$ (resp., for which $A^Tw = 0$) is $v = 0$ (resp., $w = 0$). 
\end{itemize}
\end{definition}

\begin{definition} \label{def:realRowRank} 
Let $A$ be matrix over $R$. 
\begin{itemize}
\item The {\em $R$-row rank} of $A$ (resp., the {\em $R$-column rank}), which we denote by
$$r_R(A),$$ (resp., by $r_R(A^T)$), 
is the maximal number of rows (resp., columns) that are linearly independent over $R$.  
 \item The {\em $\mathbb{R}$-row rank} (resp., the {\em $\mathbb{R}$-column rank}), which we denote
 $$r_\mathbb{R}(A),$$ (resp., $r_\mathbb{R}(A^T)$), 
 is the maximal number of rows (resp., columns) that are linearly independent over $\mathbb{R}$.
\end{itemize}
\end{definition}

\begin{remark} \label{r:rlessrho}
Since any nontrivial linear dependence relation over $\mathbb{R}$ is a nontrivial linear dependence relation over $R = \mathbb{R}[u_1, \ldots, u_D]$, it is immediate that $r_R(A) \leq r_{\mathbb{R}}(A)$.    
\end{remark}

\begin{definition}
Let $A$ be an $m \times n$ matrix over $R$. The $r$th {\em determinantal ideal} of $A$,
denoted $\mathcal{I}_r(A)$, is the ideal generated by the determinants of all $r \times r$ minors of $A$. 
\end{definition}

\begin{remark} Let $k = \min\{m,n\}$, where $A$ an $m \times n$ matrix as above. Then 
\[\mathcal{I}_1(A) \supseteq \mathcal{I}_2(A) \supseteq \ldots \supseteq \mathcal{I}_k(A) \supseteq (\mathcal{I}_{k+1}(A) = \{0\}).\]
\end{remark}

\begin{definition}
    Let $A$ be an $m \times n$ matrix over $R$. The {\em determinantal rank} of $A$
is the largest $r$ such that $\mathcal{I}_r(A) \neq \{0\}$.
\end{definition}

\begin{remark}
     Another natural notion of rank for a matrix over a commutative ring $R$ with $1$ is the 
    {\em McCoy rank} of $A$, which is defined to be the largest $r$ such that $\mathcal{I}_r(A)$ has nonzero annihilator. 
\end{remark}

\begin{definition}
The {\em annihilator} of an ideal $I$ (or any module) over $R$ is the set 
$$\textrm{ann}(I) \coloneqq \{r \in \mathbb{R} \mid r \cdot i = 0 \textrm{ for all } i \in I\} $$

\end{definition}

\begin{definition}
    A module $M$ over $R$ is said to be {\em torsion} if for all $m \in M$ there exists $(r \neq 0) \in R$ such that $r\cdot m = 0$.
\end{definition}

\begin{remark}
Since $R$ is an integral domain, then any nonzero ideal $I \subseteq R$ is necessarily non-torsion.
\end{remark}

\begin{lemma} \label{l:annihilators}
Let $R$ be an integral domain, and $I \subseteq R$ an ideal. Then $\mbox{ann}(I) \neq \{0\}$ iff $I = \{0\}$.  
\end{lemma}

\begin{proof}
If there exists $r \neq 0$ in the annihilator of $I$, then $ri = 0$ for all $i \in I$, but since $I \subseteq R$ and $R$ is an integral domain, this implies that $I = \{0\}$, as desired.

Conversely, if $I \neq \{0\},$ then there exists $i \neq 0 \in I$ such that $ri = 0$ for all $r \in \mbox{ann}(I)$, which similarly implies that $\mbox{ann}(I) = \{0\}$.
\end{proof}

\begin{lemma} \label{lem:RankNotationsIntegralDomain}
For a matrix over $R = \mathbb{R}[\theta_1,\ldots,\theta_D]$, the determinantal rank, $R$-row rank, $R$-column rank, and McCoy rank are all equal (and less than or equal to the $\mathbb{R}$-row rank, by Remark \ref{r:rlessrho}). 
\end{lemma}

\begin{proof}
It follows from the well-known Lemma \ref{l:annihilators}, combined with the less immediate results in \cite{McC} (see also \cite[Chp. 1]{McD}, \cite{MB}) that all of these notions of rank agree for matrices over an integral domain, and hence in particular for matrices over $R$.
\end{proof}

\begin{remark}
Lemma \ref{lem:RankNotationsIntegralDomain} also applies in the situation $R = \mathbb{R}$,  which is also a commutative ring with $1$; of course, if $A$ is a matrix over $\mathbb{R}$, then $\textrm{rank}(A) = r_\mathbb{R}(A^T) = r_\mathbb{R}(A)$.
Throughout the current work, we will continue to use ``rank'' to refer to the rank of a matrix over $\mathbb{R}$. For matrices over $R=\mathbb{R}[\theta_1, \ldots, \theta_D]$, we will use ``$r_R$" and ``$r_\mathbb{R}$". 
\end{remark}

For a matrix $A$ over $R = \mathbb{R}[\theta_1,\ldots,\theta_D]$, denote by $A/(\theta \to \tilde{\theta})$ the matrix formed by applying $$a_{i,j} \mapsto a_{i,j}/(\theta \to \tilde{\theta})$$ to each entry $a_{i,j}$ of $A$ (this notation was defined in equation \eqref{eq:substitutetheta}.)

\begin{lemma} \label{lem:rankdropsAlgSet}
    Let $A$ be a matrix over $R = \mathbb{R}[\theta_1,\ldots,\theta_D]$.
Then for all $\tilde{\theta} \in \mathbb{R}^D$, 
 \[\textrm{rank}\big(A/(\theta \to \tilde{\theta})\big) \leq r_R(A) \leq r_{\mathbb{R}}(A).\] Also, 
 \[
 \left\{\tilde{\theta} \in \mathbb{R}^D \mid \textrm{rank}\big(A/(\theta \to \tilde{\theta})\big) < r_R(A) \right\}
 \]
 is an algebraic subset of $\mathbb{R}^D$ of positive codimension (and hence Lebesgue measure $0$). 
\end{lemma}

\begin{proof}
It is immediate that $r_R(A) \leq r_{\mathbb{R}}(A)$ (Remark \ref{r:rlessrho}).

By Lemma \ref{lem:RankNotationsIntegralDomain}, we may interpret $r_R(A)$ to be the determinantal rank of $A$,  and we may also interpret $\textrm{rank}\big(A/(\theta \to \tilde{\theta})\big)$ to be the determinantal rank of this real matrix.

Let $M$ be an $r\times r$ minor of $A$, and assume that its determinant, $\textrm{det}(M)$, is a nonzero element of $R$. Although $\textrm{det}(M) \in R$, we will abuse notation and write $\textrm{det}(M):\mathbb{R}^D \to \mathbb{R}$ for the map represented by that polynomial. The zero locus of $\textrm{det}(M)$ consists of precisely those $\tilde{\theta} \in \mathbb{R}^D$ such that $\textrm{rank}\big(M/(\theta \to \tilde{\theta})\big) = 0$. The zero locus of a nonzero polynomial $\mathbb{R}^D \to \mathbb{R}$ is an algebraic set of positive codimension and thus Lebesgue measure $0$. 

Since there are only finitely many $r\times r$ minors $M$ of $A$, taking the union of the corresponding zero locus for each $M$ gives the result. 
\end{proof}

\begin{definition}[algebraically stable  and real rank stable sets] \label{def:stablerealrankset}
Fix $\mathcal{F}$ a piecewise polynomial parameterized function family on $\mathbb{R}^{n_0}$. Let $Z \subset \mathbb{R}^{n_0}$ be a finite set. Consider the following two conditions:
 \begin{enumerate}[(1)]
     \item \label{cond:interiorOfPiece} For every $z \in Z$, $(z,\theta)$ is in the interior of a polynomial piece for $\mathcal{F}$.
     \item \label{cond:norankdrop} $\textrm{rank}(\mathbf{J}E_Z(\theta)) = r_R(\mathbf{J}E_Z^R(\theta)).$
 \end{enumerate}
Define the \emph{algebraically stable} set for $Z$ to be the set of all $\theta \in \Omega$ such that condition \eqref{cond:interiorOfPiece} holds. 
Define the  \emph{real rank stable set} for $Z$  to be the set of all parameters $\theta \in \Omega$ for which both conditions \eqref{cond:interiorOfPiece} and \eqref{cond:norankdrop} hold.  
 \end{definition}

 \begin{remark} \label{rem:algstable}
     For a piecewise polynomial family $\mathcal{F}$, the following are equivalent:
     \begin{itemize}
         \item $(z,\theta)$ is a smooth point for $\mathcal{F}$,
         \item $(z,\theta)$ is in the interior of a polynomial piece for $\mathcal{F}$.
     \end{itemize}
     This is because two polynomials have (iterated, partial) derivatives that agree in a neighborhood of a point iff they are the same polynomial. (This consideration is also at the root of the reason why \cite{GLMW} defined batch functional dimension using only points $z$ that are parametrically smooth for a parameter $\theta$, as opposed to points $z$ so that $\mathcal{F}$ is once differentiable at $(z,\theta)$). 

Consequently, for a piecewise polynomial parameterized function family $\mathcal{F}$, the algebraically stable set coincides with the domain of the algebraic evaluation map $E_Z^R$, which is locally constant.    

In particular, condition \eqref{cond:interiorOfPiece} of Def. \ref{def:stablerealrankset} is necessary and sufficient for the quantity $\mathbf{J}E_Z^R(\theta)$ of condition \eqref{cond:norankdrop} to be defined. 
 \end{remark}

\begin{lemma} \label{l:RRstabfullinalgstab}
    Let $\mathcal{F}$ be a piecewise-polynomial parameterized function family on $\mathbb{R}^{n_0}$ with parameter space $\Omega)$.  Fix a finite set $Z \subset \mathbb{R}^{n_0}$.  Then $\mathbf{J}E_Z^R$ is constant on each connected component $C$ of the algebraically stable set for $Z$.  Moreover, the set
    $$A \coloneqq \left\{\theta \in C \mid \textrm{rank}(\mathbf{J}E_Z(\theta)) < r_R(\mathbf{J}E_Z^R\vert_C)\right\}$$
    is  contained in a closed algebraic set of positive codimension and  Lebesgue measure $0$.  Consequently the stable real rank set is open and has full measure in the algebraically stable set.    
\end{lemma}   

\begin{proof}
    It is immediate (see Remark \ref{rem:algstable}) that $\mathbf{J}E_Z^R$ is constant on each connected component $C$ of the algebraically stable set for $\mathbb{Z}$. Then there exists a fixed $r \times r$ nonsingular  minor $M$ of $\mathbf{J}E_Z^R\vert_C$ on which the $R$-rank is realized, i.e. $r_R(M) = r_R(\mathbf{J}E_Z^R\vert_C)$. 
    Therefore 
    \begin{multline*} A \subset \{\theta_0 \in C: \textrm{rank}(M / (\theta \to \theta_0)) < r_R(M)\} \\ = A \subset \{\theta_0 \in C \mid \textrm{det}(M / (\theta \to \theta_0)) = 0\}.
    \end{multline*}
    Since $\textrm{det}(M)$ is a nonzero element of the polynomial ring $R$, its zero locus is a closed algebraic set of positive codimension, and hence Lebesgue measure $0$. 
\end{proof}

\section{Growth function of difference classes and persistent pseudodimension}
\label{sec:Ppseudodim}

\subsection{Pseudo-shattering and VC dimension} \label{ss:VCdim}
For the remainder of the paper, we will restrict our attention to parameterized families $(\Omega, \mathbb{R}^{n_0}, \mathbb{R}^{n_d}, \mathcal{F})$ for which $n_d = 1$.

\subsubsection{VC dimension and pseudodimension}
Before introducing local versions of VC dimension and pseudodimension, which are well known, we recall their definitions and introduce notation we will need.

Recall that by post-composing with the sign function, $\mbox{sgn}: \mathbb{R} \rightarrow \{-1,0,1\}$, 
defined as: \[\mbox{sgn}(x) := \left\{\begin{array}{cl} \frac{x}{|x|} & \mbox{if } x \neq 0\\
x & \mbox{if } x = 0,\end{array}\right.\] any function $f: \mathbb{R}^{n_0} \rightarrow \mathbb{R}$ induces a ternary labeling on each point $z \in \mathbb{R}^{n_0}$. Explicitly, $z \in \mathbb{R}^{n_0}$ is assigned the label $\mbox{sgn}(f(z)) \in \{-1,0,+1\}$.

\begin{remark} In the classical statistical learning theory literature, cf. \cite{KearnsVazirani, Sontag}, one typically breaks symmetry by choosing a side for $0$ in the definition of the sign function. We believe the choice to define the sign function as above is more natural, and it moreover avoids some technical issues that will arise later.  
\end{remark}

Accordingly if $Z = \{z_1, \ldots, z_m\} \subseteq \mathbb{R}^{n_0}$ is a finite ordered subset of the domain, and $\mathcal{F}$ is a parameterized family, then $Z$ inherits an $m$--tuple of ternary labelings from every $f \in \mathcal{F}$. 

In the present work, we will focus our attention on the {\em binary} ($\pm 1$) labelings induced on the (necessarily open) complement of the $0$ set of a continuous function $f: \mathbb{R}^{n_0} \rightarrow \mathbb{R}$.

Let $Z = \{z_1, \ldots, z_m\}$ be a set of points in $\mathbb{R}^{n_0}$, and let $\Omega = \mathbb{R}^{D}$ be the parameter space for $\mathcal{F}$. We shall use the notation
\begin{align*}
V_{z_i}(\mathcal{F}) & \coloneqq \{\theta \in \Omega \mid F_\theta(z_i)  = 0\}, \\ 
V_Z(\mathcal{F})  & \coloneqq \{\theta \in \Omega \mid F_\theta(z_i) = 0 \mbox{ for some }i\} = \bigcup_{i=1}^m V_{z_i}(\mathcal{F}),
\end{align*}
and shall denote by $\Pi_Z(\mathcal{F}) \subseteq \{\pm 1\}^m$ the set of {\em $\pm$  labelings} achievable on $Z$ by functions in $\mathcal{F}$:
\[\Pi_Z(\mathcal{F}) := \left\{(\mbox{sgn}(F_\theta(z_1)), \ldots, \mbox{sgn}(F_\theta(z_m))) \,\, |\,\, \theta \not\in V_Z(\mathcal{F})\right\}.\]

If $|\Pi_Z(\mathcal{F})| = 2^m$ (equiv., $\Pi_Z(\mathcal{F}) = \{-1,1\}^m$ achieves its maximum possible cardinality), one says that $Z$
 is \emph{shattered} by $\mathcal{F}$. A classical measure of complexity of a parameterized function class is the growth rate of this maximal cardinality:

\begin{definition} \label{defn:growthfun} [Perles-Sauer-Shelah growth function] For a parameterized family $\mathcal{F}$ and $m \in \mathbb{N}$, define: \[\Pi_m(\mathcal{F}) := \max \{|\Pi_Z(\mathcal{F})|\,\,:\,\, |Z| = m\}.\]
\end{definition}

\begin{definition}[VC dimension]
   Let $\mathcal{F}$ be a parameterized family, and $m \in \mathbb{N}$. We say that $\mathcal{F}$ has VC dimension $m$, and write $\VCD(\mathcal{F}) = m$, if
   \begin{itemize}
       \item $\Pi_m(\mathcal{F}) = 2^m$ (i.e. some set of $m$ points is shattered by $\mathcal{F}$),
       \item $\Pi_{m+1}(\mathcal{F}) < 2^{m+1}$ (i.e. no set of $m+1$ points is shattered by $\mathcal{F})$.
   \end{itemize}
If there exists no such $m \in \mathbb{N}$, we say $\VCD(\mathcal{F}) = \infty$.
\end{definition}

The following closely-related concept was introduced by Pollard \cite{Pollard} (cf. Bartlett-Harvey-Liaw-Mehrabian \cite{BHLM}).

\begin{definition} \label{defn:pdimension} Let $\mathcal{F}$ be a parameterized family.
A finite set $Z = \{z_1, \ldots, z_k\} \subset \mathbb{R}^{n_0}$ is said to be \emph{pseudo-shattered by   $\mathcal{F}$}   if there exist thresholds $t_1, \ldots, t_k \in \mathbb{R}$ such that for every labeling $\ell:Z \to \{-1,1\}$, there exists $\theta \in \Omega$ with 
$$\mbox{sgn}(F_\theta(z_i)-t_i) = \ell(z_i)$$ for each $z_i \in Z$. 
\end{definition}

We may reframe pseudo-shattering using  thresholds $t_i = F_{\theta_0}(z_i)$ as shattering of the difference family at $\theta_0$.  The following lemma is an immediate consequence of the fact that, per Def. \ref{def:parameterizedfamily}, $\mathcal{F}$ is always assumed to be continuous. 

\begin{lemma}
A finite set $Z = \{z_i\}_{i = 1}^m  \subset \mathbb{R}^{n_0}$ is pseudo-shattered by $\mathcal{F}$ using the thresholds $t_i = F_{\theta}(z_i)$ if and only if $Z$ is shattered by the difference family $\DTheta(\mathcal{F})$ at $\theta$. 
\end{lemma}

\begin{definition}[Pseudo-shattering at a parameter, VC dimension of difference classes] \label{defn:pshatter}
Let $\mathcal{F}$ be a parameterized family, and $\theta_0 \in \Omega$ a fixed parameter. 

\begin{enumerate} 
\item 
We say that a finite set $Z = \{z_1, \ldots, z_m\} \subset \mathbb{R}^{n_0}$ is  \emph{pseudo-shatterered by $\mathcal{F}$ at $\theta_0$} if $|\Pi_Z(\DTheta)| = 2^m$.
\item We say that $\mathcal{F}$ has {\em pseudodimension $m$} relative to $\theta_0$, and write \[\pVCD(\mathcal{F}, \theta_0) = m\] if the VC dimension of the difference class $\DTheta(\mathcal{F})$ is $m$ (that is, if:
\begin{itemize}
    \item $\Pi_{m}(\DTheta) = 2^m$,
    \item $\Pi_{m+1}(\DTheta) < 2^{m+1}$).
\end{itemize}
\end{enumerate}
\end{definition}

\subsubsection{Persistent pseudodimension}

We now introduce new \emph{local} versions of the notions above. (Compare the notion of local Rademacher complexity, \cite{BBM}.) 

\begin{definition}[persistently pseudoshattered, persistent binary capacity, persistent growth function] \label{defn:locpgrowth}
Let $\mathcal{F}$ be a parameterized family and $\theta_0 \in \Omega$. Let $Z \subset \mathbb{R}^{n_0}$ be a finite set. 
\begin{itemize}
    \item We will say that $Z$ is \emph{persistently pseudoshattered} at $\theta_0$ if \[2^{|Z|} = \lim_{\epsilon \searrow  0} |\Pi_Z(\DepsilonTheta)|.\]
    \item We denote by $\psi_Z(\theta_0)$ the \emph{maximal cardinality of a persistently pseudoshattered subset of $Z$}:
    \[\psi_Z(\theta_0) \coloneqq \max \left\{|\widetilde{Z}| : \widetilde{Z} \subset Z \textrm{ is persistently pseudoshattered at }\theta_0\right\}.\]
    \item We define the {\em persistent binary capacity} on $Z$ at $\theta_0$ to be: \[\aleph_Z(\theta_0) := \lim_{\epsilon \searrow  0} |\Pi_Z(\DepsilonTheta)|.\]
    \item We define the {\em persistent Perles-Shelah-Sauer growth function} at $\theta_0$, denoted $\aleph_m(\theta_0)$, to be  the maximal possible persistent binary capacity near $\theta_0$ over all batches $Z \subset \mathbb{R}^{n_0}$ of size $m$: 
\[\aleph_m(\theta_0) := \max_{|Z|=m} \left\{\lim_{\epsilon \searrow  0} |\Pi_Z(\DepsilonTheta)|\right\} = \max_{|Z| = m} \left\{\aleph_Z(\theta_0)\right\}.\]   
\end{itemize}
\end{definition}

\begin{remark} \label{rmk:strict} It is immediate from the definitions that \[\psi_Z(\theta_0) \leq \lfloor\log_2(\aleph_Z(\theta_0))\rfloor,\] but it is quite possible for this inequality to be strict, for example if (for sufficiently small $\epsilon$)  $\Pi_Z(\mathcal{D}_{\theta_0, \epsilon})$ contains $2^k$ sign sequences for which no projection to a $k$--face of the Hamming cube, $[-1,1]^{|Z|}$, is surjective onto the vertices of the $k$--face. As a concrete example, let $|Z|=3$ and \[\Pi_Z(\mathcal{D}_{\theta_0,\epsilon}) = \{(-1,-1,-1), (+1,-1,-1), (+1,-1,+1), (+1,+1,+1)\}.\] Then $\psi_Z(\theta_0) = 1,$ since the maximal size of a persistently pseudoshattered set is $1$, but $\log_2(\aleph_Z(\theta_0)) = 2.$
    
\end{remark}

\begin{remark} Note that  \[|\Pi_Z(\DepsilonTheta)|\leq |\Pi_Z(\mathcal{D}_{\theta_0, \epsilon'})| \leq |\Pi_Z(\DTheta)| \leq 2^m\] for all $0 < \epsilon < \epsilon'$, so if $Z$ is persistently pseudo-shattered at $\theta_0$, then $Z$ is pseudo-shattered at $\theta_0$. Equivalently, if $Z$ is {\em not} pseudo-shattered at $\theta_0$, then $Z$ is {\em not} persistently pseudo-shattered at $\theta_0$.    
\end{remark}

\medskip

\begin{remark} \label{rem:localextreme} Let \[\mathcal{L}: \Omega \times (\mathbb{R}^{n_0})^m \rightarrow \mathbb{R}\] be a parameterized family of empirical losses for batches of size $m$ for a family $\mathcal{F}$. In this case, if $\theta_0$ is a global (resp., local) minimum or maximum for the empirical loss $\mathcal{L}_Z$ relative to a batch $Z = \{z_1, \ldots, z_m\}$, then $Z$ must not pseudo-shatter $\mathcal{L}_Z$ at $\theta_0$ (resp., must not persistently pseudo-shatter $\mathcal{L}_Z$ at $\theta_0$), since all differences, $\mathcal{L}_Z(\theta) - \mathcal{L}_Z(\theta_0)$, necessarily have the same sign. Pseudo-shattering at $\theta_0$ (resp., persistent pseudo-shattering at $\theta_0$) can therefore be viewed as an obstruction to $\theta_0$ being a global (resp., local) minimum of the empirical loss landscape. 
\end{remark}

\begin{definition}[persistent pseudodimension] \label{defn:locpdim}
Let $\mathcal{F}$ be a parameterized family, and $\theta_0 \in \Omega$ a fixed parameter. We say that $\mathcal{F}$ has {\em persistent pseudodimension $m$ at $\theta_0$}, and write $\ppVCD(\mathcal{F}, \theta_0) = m$, if 
\begin{enumerate} 
\item $\aleph_m(\theta_0) = 2^m$ (i.e. some batch of $m$ points is persistently pseudo-shattered at $\theta_0$), and 
\item 
$\aleph_{m+1}(\theta_0) < 2^{m+1}$ (i.e. no batch of $m+1$ points is persistently pseudo-shattered at $\theta_0).$  
\end{enumerate}
\end{definition}

\begin{remark}[batch persistent pseudodimension] \label{r:batchppD} Given a parameterized family $\mathcal{F}$ and a fixed parameter $\theta_0 \in \Omega$, it is immediate from the definitions that the supremum of $\psi_Z(\theta_0)$ over finite subsets $Z \subset \mathbb{R}^{n_0}$ of the domain is the persistent pseudodimension of $\mathcal{F}$ at $\theta_0$: \[\sup_{\mbox{\footnotesize{finite }} Z \subseteq \mathbb{R}^{n_0} } \psi_Z(\theta_0) = \ppVCD(\mathcal{F},\theta_0).\] Accordingly, we will sometimes refer to $\psi_Z(\theta_0)$ as the {\em batch persistent pseudodimension} of $\mathcal{F}$ at $\theta_0$ for the batch $Z$. 
\end{remark}

The Sauer-Shelah Lemma, a foundational result in the theory of VC dimension, states that for any  family $\mathcal{F}$ and $m \in \mathbb{N}$, 
$$|\Pi_m(\mathcal{F})| \leq \sum_{i=0}^{dim_{VC}(\mathcal{F})} {m \choose i}.$$

\begin{lemma}[persistent version of Sauer-Shelah Lemma] \label{l:persistentSauer}
$$\aleph_m(\theta_0) \leq \sum_{i=0}^{\ppVCD(\mathcal{F},\theta_0)} {m \choose i}.$$
\end{lemma}

\begin{remark}
One might expect that Lemma \ref{l:persistentSauer} is an immediate corollary of Sauer's Lemma; however, a subtle issue involving limits as $\epsilon \searrow 0$ in the persistent setting seems to prevent this immediate implication. 

What the Sauer-Shelah Lemma gives us is that for any fixed $\epsilon > 0$, for any set $Z$ of $m$ points, 
$$|\Pi_Z(\DepsilonTheta)| \leq \sum_{i=0}^{dim_{VC}(\DepsilonTheta)} {m \choose i},$$
and hence 
$$\aleph_m(\theta_0)\leq \lim_{\epsilon \searrow 0}\sum_{i=0}^{dim_{VC}(\DepsilonTheta)} {m \choose i} = 
\sum_{i=0}^{\lim_{\epsilon \searrow 0}dim_{VC}(\DepsilonTheta)} {m \choose i}
$$
The issue is that in the quantity at right above, for each $\epsilon > 0$, $dim_{VC}(\DepsilonTheta)$ may be realized on \emph{different} finite subset of $\mathbb{R}^{n_0}$; in contrast,$\ppVCD(\mathcal{F},\theta_0)$ measures the number of sign patterns achieved on some \emph{fixed} set by $(\DepsilonTheta)$ for all $\epsilon > 0$. 
Without additional restrictions on the parameterized family, it could be the case that $\lim_{\epsilon \searrow 0} dim_{VC}(\DepsilonTheta) > {\ppVCD(\mathcal{F},\theta_0)}$.

This issue also prevents a well-known proof technique for the Sauer-Shelah lemma (an induction argument based on an inclusion/exclusion principle) from applying directly to the persistent setting. Fortunately, the 
``shifting'' proof of Sauer's original result, discovered independently in \cite{Alon, Frankl}, does carry over to the persistent setting. The proof below sketches the shifting argument to verify that it carries over directly. 
\end{remark}

\begin{proof}
Recall the definition 
\[\aleph_k(\theta_0) \coloneqq \max_{|Z|=k} \left\{\lim_{\epsilon \searrow  0} |\Pi_Z(\DepsilonTheta)|\right\}. \]
For convenience, write $d=\ppVCD(\mathcal{F},\theta_0)$.
Fix $m > d$.  Fix any set $Z=\{z_1,\ldots,z_m\} \subset \mathbb{R}^{n_0}$ of $m$ points. Then since $\aleph_{d+1}(\theta_0) < 2^{d+1}$, we may fix $\epsilon > 0$ such that no $(d+1)$-element subset of $Z$ is shattered by  $\DepsilonTheta$.

For $\theta \in B_{\epsilon}(\theta_0)$, let $h_{\theta}:Z \to \{\pm 1\}$ be the function defined by $h(z_i) = \textrm{sgn}(\mathcal{D}_{\theta_0}(\theta,z_i))$.  Set $\mathcal{H} = \{h_\theta \mid \theta \in B_{\epsilon}(\theta_0)\}.$ Since $\mathcal{H}$ is a finite set, we may list its elements, say $\mathcal{H} = \{h_1,\ldots,h_n\}$.
Let $M$ be the $n \times m$ matrix (i.e. $n$ rows and $m$ columns) whose $(j,i)$th entry is $M_{j,i} = h_j(z_i)$.  

We now directly apply the ``shifting'' technique of \cite{Alon, Frankl} to $M$, obtaining that $M$ has at most $\sum_{i=0}^d {m \choose i}$ rows.  For completeness, we sketch the outline of this argument.

\medskip 
\noindent \textbf{Shifting Algorithm:}

\begin{itemize}
    \item \noindent Step 1:

For $1 \leq i \leq m$,

\hspace{.5cm} For $1 \leq k < i$,

 \hspace{1cm} For $1 \leq j \leq n$, 
 
\hspace{1.5cm}  If $M_{j,i} = +1$ and $M_{j,k} = -1$ and swapping the values of $M_{j,i}$ and $M_{j,k}$ doesn't duplicate a row of the matrix,

\hspace{2cm} then swap the values of $M_{j,i}$ and $M_{j,k}$.

\item \noindent Step 2: Repeat step 1 until no further change is possible. 
\end{itemize}
\medskip

\noindent Let $M'$ be the matrix that is the output of the shifting algorithm. One can show (see \cite{Alon, Frankl} or more recent expositions such as \cite{Gerbner}
for details) $M'$ has the following properties:
\begin{enumerate}
 \item If some row of $M'$ has $+1$s in positions $i_1,\ldots,i_r$, then there must be $2^r-1$ other rows of $M'$ that have all combinations of $\pm 1$s in columns $i_1,\ldots,i_r$, i.e. the set $\{z_{i_1},\ldots,z_{i_r} \}$ is shattered by (the rows of) $M'$.  
  \item If some subset $V \subset Z$ is shattered by the rows of $M'$, then $V$ was also shattered by the rows of the original $M$. 
\end{enumerate}
It follows that each row of $M'$ has at most $d$ entries that are $+1$.  Since the rows of $M'$ are all distinct, this implies $M'$ has at most $\sum_{i=0}^d {m \choose i}$ rows.  Thus $M$ does too.
\end{proof}

\subsection{Persistent pseudo-shattering \& parametric smoothness for ReLU network functions}

While the definitions and constructions in section \ref{ss:VCdim} are applicable to any continuous parameterized function class, we now state a result that is specific to the class of fully-connected feedforward ReLU neural network functions.

The following proposition tells us that if there exists a persistently pseudo-shattered set $Z'$ for a parameterized class of feedforward ReLU network functions, then we can find a persistently pseudo-shattered set $Z$ with $|Z| = |Z'|$ that is also parametrically smooth. We need this because we would like to compare local functional dimension with persistent pseudodimension, and computing local functional dimension requires parametric smoothness of a batch, $Z$.

\begin{proposition} \label{prop:ParSmooth} Let $\mathcal{F}$ be any standardly-parameterized function class of ReLU neural network functions of architecture $(n_0, \ldots, n_{d-1}| n_d)$, with $n_0 \geq 1$, and $n_d = 1$. Let $\theta_0 \in \Omega$ be generic and supertransversal. For each $m \in \mathbb{N}$, \[\aleph_m(\theta_0) = \max_{|Z| = m}\{\aleph_Z(\theta_0)\} = \max_{|Z| = m}\left\{\lim_{\epsilon \searrow 0}|\Pi_{Z}(\mathcal{D}_{\theta_0,\epsilon})|\right\} \] is realized on a parametrically smooth set $Z$.
\end{proposition}

\begin{proof} Let $Z' \subset \mathbb{R}^{n_0}$ be a set of size $m$ that achieves the maximum number of persistent $\pm$ sign assignments. If every point in $Z'$ is contained in a $\pm$ activation region for $\theta_0$, the set $Z'$ is parametrically smooth by \cite[]{GLMW}.  So we may assume without loss of generality that there is at least one point $z' \in Z'$ whose ternary labeling for $\theta_0$ has at least one $0$. Moreover, if the ternary labeling, $s_{z'} \in \{-1,0,1\}^{N}$ of $z'$ has $k$ $0$'s, then $z'$ is contained in a {\em codimension $k$} cell, $C_{z'}$, of the canonical polyhedral complex $\mathcal{C}_{\theta_0}$ and by \cite[]{Masden}, $C_{z'}$ is a nontrivial face of $2^k$ non-empty cells of $\mathcal{C}_{\theta_0}$ whose ternary labelings are obtained from the ternary labeling of $C_{z'}$ by replacing one or more of the $0$'s in $s_{z'}$ with a $\pm 1$.

Choose $\epsilon' > 0$ sufficiently small to ensure that \[|\Pi_{Z'}(\mathcal{D}_{\theta_0,\epsilon'})| = \lim_{\epsilon \searrow 0}|\Pi_{Z'}(\DepsilonTheta)|.\] Now for each distinct $\pm$ activation pattern $s \in \Pi_{Z'}(\mathcal{D}_{\theta_0,\epsilon'}(\mathcal{F}))$ achieved on $Z'$ fix a $\theta_s \in \mathcal{D}_{\theta_0}(\mathcal{F},\epsilon')$ realizing that activation pattern. By continuity of the function $F_{\theta_s} - F_{\theta_0}: \mathbb{R}^{n_0} \rightarrow \mathbb{R}$, there is an open neighborhood, $N_s$, of $z'$ on which $F_{\theta_s} - F_{\theta_0}$ achieves the same sign as it does on $z'$. Moreover, since $\theta_0$ is generic and supertransversal, the subset of $\Omega$ consisting of parametrically smooth points for $\theta_0$ is full measure. It follows that the subset of $N = \bigcap_{s} N_s$ consisting of parametrically smooth points for $\theta_0$ is full measure in $N$. 

We now see that if we replace $z' \in Z'$ with any such parametrically smooth point $z \in N$: \[Z'' := \{z\} \cup (Z' \setminus \{z'\}),\] we have \[\Pi_{Z'}(\mathcal{D}_{\theta_0,\epsilon'}) \subseteq \Pi_{Z''}(\mathcal{D}_{\theta_0,\epsilon''})\] for some $\epsilon''$ with $0 < \epsilon'' < \epsilon'$.  But since \[\Pi_{Z'}(\mathcal{D}_{\theta_0, \epsilon'}) = \Pi_{Z'}(\mathcal{D}_{\theta_0,\epsilon''}) = \max_{|Z| = m}\left\{\lim_{\epsilon \searrow 0}|\Pi_{Z}(\mathcal{D}_{\theta_0,\epsilon})|\right\},\] we see that in fact $\Pi_{Z'}(\mathcal{D}_{\theta_0,\epsilon'}) = \Pi_{Z''}(\mathcal{D}_{\theta_0,\epsilon''})$. By repeating this process in turn for each (of the finitely many) remaining non-parametrically smooth point(s) in $Z''$, the result follows.
\end{proof}

\begin{corollary}
    \label{cor:ParSmoothPsi}
Let $\mathcal{F}$ be any standardly-parameterized function class of ReLU neural network functions of architecture $(n_0, \ldots, n_{d-1}| 1)$, with $n_0 \geq 1$, and let $\theta_0 \in \Omega$ be generic and supertransversal. Let ${\bf Z}$ be the set of finite subsets of the domain $\mathbb{R}^{n_0}$ that are parametrically smooth for $\theta_0,$ and let ${\bf Z}'$ be the set of finite subsets of $\mathbb{R}^{n_0}$ with no additional conditions imposed. Then \[\sup_{Z \in {\bf Z}'} \psi_Z(\theta_0) = \sup_{Z \in {\bf Z}} \psi_Z(\theta_0).\]
\end{corollary}

\begin{proof} Since ${\bf Z}' \supseteq {\bf Z}$, it is immediate that \[\sup_{Z \in {\bf Z}'}\psi_Z(\theta_0) \geq \sup_{Z \in {\bf Z}}\psi_Z(\theta_0).\] Now suppose, aiming for a contradiction, that \[\sup_{Z \in {\bf Z}'} \psi_Z(\theta_0) > \sup_{Z \in {\bf Z}}\psi_Z(\theta_0).\] Then there exists some non-parametrically smooth finite set $Z'$ for which \[\psi_{Z'}(\theta_0) > \sup_{Z \in {\bf Z}'}\psi_Z(\theta_0) \geq \psi_{Z}(\theta_0)\] for all parametrically smooth sets $Z$. Consider the subset $Z'' \subseteq Z'$ satisfying $|Z''| = \psi_{Z'}(\theta_0).$ By the observation above, $Z''$ must not be a parametrically smooth set. 

Since $Z''$ is persistently pseudoshattered, letting $m := |Z''|$, we conclude that \[\aleph_m(\theta_0) = \max_{Z \in {\bf Z}'} \aleph_Z(\theta_0) = 2^m,\] and the maximum is achieved on $Z''$. By assumption there is no parametrically smooth set $Z$ for which $\psi_Z(\theta_0) = \psi_{Z''}(\theta_0)$. But this contradicts Proposition \ref{prop:ParSmooth}. We conclude that 
\[\sup_{Z \in {\bf Z}'}\psi_Z(\theta_0) \leq \sup_{Z \in {\bf Z}}\psi_Z(\theta_0),\] and hence \[\sup_{Z \in {\bf Z}'}\psi_Z(\theta_0) = \sup_{Z \in {\bf Z}}\psi_Z(\theta_0),\] as desired.
\end{proof}

\section{Constant rank maps}
\label{sec:BatchFiber} 

\subsection{The local batch fiber product structure}

The classical Rank Theorem for smooth manifolds (Theorem 4.12 of \cite{LeeSmoothManifolds}) says, roughly speaking, that given a smooth map $F$ of constant rank $r$ between between manifolds, there exist smooth coordinates with respect to which $F$ is projection onto the first $r$ coordinates.  The proposition below is a variant of the Rank Theorem, adapted to our situation; we prove this variant because we will need not only the existence of the smooth coordinates, but an explicit description of those smooth coordinates in terms of the canonical Euclidean coordinates on $\Omega = \mathbb{R}^D$ and $\mathbb{R}^{|Z|}$.

\begin{proposition} \label{localdiffeo}
Let $\mathcal{F}$ be a parameterized piecewise-polynomial family on $\mathbb{R}^{n_0}$ with parameter space $\Omega$.  Fix a finite set $Z = \{z_1, \ldots, z_m\} \subset \mathbb{R}^{n_0}$. 
Let $\theta_0$ be a parameter in the real rank stable set for $Z$ (Def. \ref{def:stablerealrankset}) and set $r = r_R(\mathbf{J}E^R_Z(\theta_0))$. 
Let $M$ be a nonsingular $r \times r$ minor of $\mathbf{J}E^R_Z(\theta_0))$, let 
$\widetilde{Z} \subset Z$ be the points that correspond to the rows of $M$ and  let $S \subset \{1,\ldots,D\}$ be the indices that correspond to the columns of $M$. Set  $N = Z \setminus \widetilde{Z}$ and $T = \{1,\ldots,D\} \setminus S$. 
 Then there exist
\begin{itemize}
    \item an open neighborhood $V \subset \mathbb{R}^r \times \mathbb{R}^{D-r} \cong \Omega$ of $\theta_0$ (on which we will use the coordinate $v=(v^S,v^T)$),
    \item an open set $W \subset \mathbb{R}^r \times \mathbb{R}^{D-r}$ (on which we will use the coordinate $w = (w^S,w^T)$),
    \item and a diffeomorphism $\psi: W \to V \subset \Omega$
    \end{itemize}
    such that the following hold for all $w=(w^S,w^T) \in W$. 
 \begin{enumerate} 
\item \label{i:cov1} $E_Z \circ \psi(w) = (E_{\widetilde{Z}}(\psi(w)), E_N(\psi(w)) = 
(w^S,E_N(\psi(w))).$
\item \label{i:cov2} $E_Z \circ \psi(w^S,w^T)$ does not depend on $w^T$, i.e. 
$$ \frac{\partial E_N(\psi(w^S,w^T))}{\partial w^T} = [0].$$
\end{enumerate}
\end{proposition}

Before proceeding to the proof, let us briefly unpack the meaning of the Proposition \ref{localdiffeo}.  The diffeomorphism $\psi:W \to V$ is a change of coordinates.  $V \subset \Omega$, equipped with coordinates $v = (v^S,v^T)$, is a neighborhood of $\theta_0$ in our ``original'' coordinates (with respect to which the evaluation map $E_Z$ is defined); we will find it beneficial to instead use $W$, equipped with the coordinates $w=(w^S,s^T)$.  
Since $Z = \widetilde{Z} \sqcup N$, it is immediate that $E_Z \circ \psi(w) = (E_{\widetilde{Z}}(\psi(w)), E_N(\psi(w))$.  The meaningful part of item \ref{i:cov1} is the rightmost equality, which says that $E_{\widetilde{Z}} \circ \psi$ gives the first $r$ places of the (diffeomorphic) $w$ coordinate system (the $w^S$ part of $w=(w^S,w^T)$.
Item \ref{i:cov2} says that moving in the $w^T$ directions doesn't change $E_Z \circ \psi$.

\begin{proof}[Proof of Proposition \ref{localdiffeo}]
Since $\theta_0$ is in the real rank stable set for $Z$ and $M$ is a minor of $\mathbf{J}E_Z^R(\theta_0)$ on which the rank is realized, 
 $$\textrm{rank}(\mathbf{J}E_Z(\tilde{\theta})) = \textrm{rank}(M / (\theta \to \tilde{\theta}))= r_R(M) = r_R(\mathbf{J}E_Z^R\vert_U). $$
Without loss of generality (by reordering the coordinates, as necessary), we may assume $S = \{1,\ldots,r\}$ and $\widetilde{Z}$ is the first $r$ elements of $Z$ (so $M/(\theta \to \tilde{\theta})$ is the top left block of $\mathbf{J}E_Z(\tilde{\theta})$). 

\bigskip

Let $U$ be an open neighborhood of $\theta_0$ contained in the real rank stable set for $Z$. 
Define $\phi: U \subset \mathbb{R}^r \times \mathbb{R}^{D-r} \to \mathbb{R}^r \times \mathbb{R}^{D-r}$ by
$$\phi(v) = \phi(v^S,v^T) = \left(E_{\widetilde{Z}}(v), v^T\right).$$

Then 
\[\mathbf{J}\phi(v) =
\begin{bmatrix}
\frac{\partial E_{\widetilde{Z}}(v)}{\partial v^S}& \frac{\partial E_{\widetilde{Z}}(v)}{\partial v^T}\\
\\
\frac{\partial v^T}{\partial v^S}  & \frac{\partial v^T}{\partial v_T}\\
\end{bmatrix} = 
\begin{bmatrix}
\frac{\partial E_{\widetilde{Z}}(v)}{\partial v^S}& \frac{\partial E_{\widetilde{Z}}(v)}{\partial v^T}\\
\\
0_{(D-r) \times r}  & I_{(D-r) \times (D-r)}\\
\end{bmatrix}.\] 
We know
$\frac{\partial E_{\widetilde{Z}}(v)}{\partial v^S}$ is a nonsingular $r \times r$ minor.  $I_{(D-r)\times (D-r)}$ is also nonsingular.  Therefore, the entire matrix 
$\mathbf{J}\phi(u)$ is nonsingular. Hence the Inverse Function Theorem guarantees that there exists an open neighborhood $V \subset U$ of $\theta_0$ and an open set $W$ such that 
 $\phi : V \to W$  is a diffeomorphism. Define $\psi: W \to V$ by $\psi = \phi^{-1}$. 

By construction, for $w=(w^S,w^T) \in W$, 
$$E_Z \circ \psi(w) = (E_{\widetilde{Z}}(\psi(w)), E_N(\psi(w)) = 
(w^S,E_N(\psi(w))).$$
Therefore 
\[
\mathbf{J}(E_Z \circ \psi)(w^S,w^T) = 
\begin{bmatrix}
\frac{\partial w^S}{\partial w^S}& \frac{\partial w^S}{\partial w^T}\\
\\
\frac{\partial E_N(\psi(s))}{\partial w^S}  & \frac{\partial E_N(\psi(w))}{\partial w^T}\\
\end{bmatrix} = 
\begin{bmatrix}
I_{r \times r}& 0_{r \times (D-r)}\\
\\
\frac{\partial E_N(\psi(w))}{\partial w^S}  & \frac{\partial E_N(\psi(w))}{\partial w^T}\\
\end{bmatrix}
\]
Since precomposing with a diffeomorphism does not change the rank of a map, the composition $E_Z \circ \psi : W \to \mathbb{R}^r \times \mathbb{R}^D$ has constant rank $r$.

The left $r$ columns on the rightmost matrix above are clearly linearly independent.  So, since  the rank of the total matrix is $r$, we must have that 
$$ \frac{\partial E_N(\psi(w))}{\partial w^T} = [0]_{(|Z|-r) \times( D-r)}.$$
\end{proof}

Given a parameterized family $\mathcal{F}:\Omega \times \mathbb{R}^{n_0} \to \mathbb{R}^{n_d}$, the \emph{fiber} (with respect to the realization map $\rho$) of a function $f:\mathbb{R}^{n_0} \to \mathbb{R}^{n_d}$ is 
\[\rho^{-1}(f) \coloneqq \{\theta \in \Omega \mid F_{\theta} = f\}.\]
Some preliminary observations about the structure of fibers for ReLU neural network families were presented in \cite{GLMW}.  We now introduce the definition of a \emph{batch fiber}.

\begin{definition} \label{def:batchfiber}
Fix a parameterized family $\mathcal{F}:\Omega \times \mathbb{R}^{n_0} \to \mathbb{R}^{n_d}$ and fix a set $Z \subset \mathbb{R}^{n_0}$.  We define the $Z$-\emph{batch fiber} of a function $f:\mathbb{R}^{n_0} \to \mathbb{R}^{n_d}$ to be the set of parameters $\theta \in \Omega$ for which the restrictions to $Z$ of $F_{\theta}$ and $f$ coincide:
\[\rho_Z^{-1}(f) \coloneqq \{\theta \in \Omega \mid F_{\theta}\vert_Z = f\vert_Z\}.\]
\end{definition}

Equipped with the definition of a batch fiber, we can now interpret Proposition \ref{localdiffeo} as saying that, locally, the real rank stable set for $Z$ has a product structure -- locally it is the product of the ($r$-dimensional) batch fibers and $[0,1]^{D-r}$.

 \begin{theorem}[Batch fiber product structure] \label{t:batchfiberfoliation}
Let $\mathcal{F}$ be a piecewise polynomial parameterized family, 
let $Z= \{z_1,\ldots,z_{k} \} \subset \mathbb{R}^{n_0}$ be a finite set. Let $\theta$  be a parameter in the real rank stable set for $Z$, and set $r = \bFD(\theta,Z)$. Then there 
exist an open neighborhood $V$ of $\theta$ and a diffeomorphism 
$$\psi: B^{r} \times B^{D-r} \to V$$
such that each ``slice'' $\psi(\{x\} \times B^{D-r})$ is the intersection of a $Z$-batch fiber with $V$. 
\end{theorem}

\begin{proof}
Let $\psi:W \to V \owns \theta$ be the diffeomorphism from Proposition \ref{localdiffeo}.   For any fixed value $x$ of $w_s$, let $W_x = \{(x,w^T) \in W\}$; then $\psi(W_x)$ is a smooth manifold of codimension $r=\bFD(\theta,Z)$.  Item \eqref{i:cov2} of Proposition \ref{localdiffeo} guarantees that $E_Z$ is constant on each manifold $\psi(W_x)$, i.e. $\psi(W_x)$ is contained in a single batch fiber for $Z$. Item \eqref{i:cov1} guarantees that if $x_1 \neq x_2$, then $E_Z(\psi(W_{x_1})) \neq E_Z(\psi(W_{x_2}))$, so $\psi(\{x_1\} \times B^{D-r})$ and $\psi(\{x_2\} \times B^{D-r})$ are in different batch fibers. 
\end{proof}

Combining Theorem \ref{t:batchfiberfoliation} with Proposition \ref{p:stabRrksetbig} immediately yields the following corollary. 

\begin{corollary}[a.e. batch fiber structure for ReLU networks]
Let $\mathcal{F}$ be parameterized family of ReLU networks of architecture $(n_0,\ldots,n_d\vert 1)$. Fix a nonempty finite set $Z \subset \mathbb{R}^{n_0}$.  Then for Lebesgue almost every parameter $\theta \in \Omega$, there exists an open neighborhood $V$ of $\theta$ and a diffeomorphism 
$$\psi: B^{r} \times B^{D-r} \to V$$
such that each ``slice'' $\psi(\{x\} \times B^{D-r})$ is the intersection of a $Z$-batch fiber with $V$. 
\end{corollary}

We conclude this section with some immediate consequences of Theorem \ref{t:batchfiberfoliation} 
for the loss landscape and gradient descent.  Suppose $Z \subset \mathbb{R}^{n_0}$ is the set of all inputs of the testing data points. 
Let us first define what we mean by a \emph{loss function} (with regards to the test set $Z$): 
a function $L_Z: \Omega \to \mathbb{R}^{\geq 0}$ such that $E_Z(\theta_1) = E_Z(\theta_2)$ implies $L(\theta_1) = L(\theta_2).$

\begin{lemma}[Immediate consequences of the batch fiber product structure]
    Let $\mathcal{F}$ be a piecewise polynomial parameterized family and let $Z \subset \mathbb{R}^{n_0}$ be a finite set. Then for any parameter $\theta$ in the real rank stable set for $Z$ the following hold. 
    
    \begin{enumerate}
       \item On an open neighborhood of $\theta$, level sets of a loss function $L_Z$ are unions of $Z$-batch fibers.
       Furthermore, locally near $\theta$, a level set for $L_Z$ is a $\bFD(\theta,Z)$-dimensional smooth manifold. 
    \item The gradient of $L_Z$ at $\theta$ is constrained to lie in the $(D-\bFD(Z,\theta)$-dimensional  orthogonal complement of the tangent space at $\theta$ of the batch fiber at $\theta$, i.e.
    \[\nabla L_Z(\theta) \in T_{\theta}(\rho_Z^{-1}(F_\theta))^{\perp}.\]
    \end{enumerate}
\end{lemma}

\section{Batch persistent pseudodimension: bounds and the rank gap}
\label{sec:RankGap}
\subsection{Bounds}

Recall (Def. \ref{defn:locpgrowth}) that for a parameterized family $\mathcal{F}$ and a finite subset of the domain, $Z \subseteq \mathbb{R}^{n_0}$, we denote by $\psi_Z(\theta_0)$ the maximal cardinality of a persistently pseudoshattered subset of $Z$  at $\theta_0 \in \Omega$.  We also refer to this quantity as the batch persistent pseudodimension of $\theta_0$ on the batch $Z$ (Remark \ref{r:batchppD}).

In this section, we establish both upper and lower bounds on the batch persistent pseudodimension of $\theta_0$ using the notions of rank developed in Section \ref{sec:Rkpoly}.

\begin{theorem} \label{t:inequalitystring}
Let $\mathcal{F}$ be a parameterized piecewise-polynomial family, and let $Z = \{z_1, \ldots, z_m\} \subset \mathbb{R}^{n_0}$ be a finite set. 
Let $\theta_0$ be a parameter in the algebraically stable set for $Z$.  
Then there is an open neighborhood $U$ of $\theta_0$ such that for all $\theta \in U$, 
\begin{equation} \label{eq:bounds}
  \bFD(\theta_0,Z) \leq \psi_Z(\theta) \leq r_{\mathbb{R}}({\bf J}E_Z^R(\theta)), 
\end{equation}
and for almost all $\theta \in U$,
\begin{equation} \label{eq:boundsRrank}
  r_R({\bf J}E_Z^R(\theta)) \leq \psi_Z(\theta) \leq r_{\mathbb{R}}({\bf J}E_Z^R(\theta)).    
\end{equation}
\end{theorem}

\begin{remark}
    If $\theta_0$ is not only algebraically stable for $Z$ but also 
    in the stable real rank set,  then $\textrm{dim}_{\textrm{ba.fun}}(\theta_0,Z) = \textrm{dim}_{\textrm{ba.fun}}(\theta,Z) = r_R({\bf J}E^R_Z(\theta))$ for all $\theta$ in an open neighborhood of $\theta_0$, so equations \eqref{eq:bounds} and \eqref{eq:boundsRrank} become 
      \[\left\{\bFD(\theta,Z) = r_R({\bf J}E_Z^R(\theta))\right\} \leq  \psi_Z(\theta) \leq r_{\mathbb{R}}({\bf J}E_Z^R(\theta))\]  for all $\theta$ in some open neighborhood of $\theta_0$. 
\end{remark}

\begin{proof}
    Proposition \ref{p:lowerbound} gives $\bFD(\theta_0,Z) \leq 
 \psi_Z(\theta)$. Lemma \ref{l:upperbound} implies $\psi_Z(\theta) \leq r_{\mathbb{R}}(\mathbf{J}E_Z^R(\theta)).$ The final statement and string of inequalities in Equation \ref{eq:boundsRrank} follows from the fact that for any finite set $Z$ the stable real rank set is full measure in the algebraically stable set, and by definition $\bFD(\theta,Z) = r_R({\bf J}E^R_Z(\theta))$ if $\theta$ is in the real rank stable set for $Z$.
\end{proof}

\begin{proposition} \label{p:lowerbound}
Let $\mathcal{F}$ be a parameterized piecewise-polynomial family on $\mathbb{R}^{n_0}$ with parameter space $\Omega$ and fix $\theta_0 \in \Omega$.  Let $Z = \{z_1, \ldots, z_m\} $ be a set of parametrically smooth points for $\theta_0$ (equivalently, $\theta_0$ is in the algebraically stable set for $Z$).  Then there is an open neighborhood $U$ of $\theta_0$ such that $$\textrm{dim}_{\textrm{ba.fun}}(\theta_0,Z) \leq \psi_Z(\theta)$$ for all $\theta \in U$. 
\end{proposition}

\begin{proof}
Set 
$r \coloneqq r_R(JE_Z^R(\theta_0)).$
By algebraic stability, we also have that $r = r_R(JE_Z^R(\theta))$ for all parameters $\theta$ in some small closed ball $\overline{B} \subset \Omega$ of strictly positive radius around $\theta_0$.  
Set $$r' \coloneqq \textrm{rank}(JE_Z(\theta_0)) = \textrm{dim}_{\textrm{ba.fun}}(\theta_0,Z).$$

It is immediate that $r' < r$.  For each nonnegative integer $k < r$, the set $$K_k \coloneqq \{\theta \in \overline{B} \mid \textrm{rank}(JE_Z(\theta)) = k\}$$
is compact.  Then distance to $\theta_0$ is a continuous function on the compact set $K_k$, hence attains its minimum.  Hence, for each $k < r'$, $K_k$ is a positive distance from $\theta_0$, i.e. there is some open ball $B' \subset \overline{B}$ about $\theta_0$ such that 
$$\textrm{rank}(JE_Z(\theta)) \geq r'$$
for all $\theta \in B'$. 
By a similar argument, we may assume that there is a fixed set of $r'$ elements of $Z$, call it $Y=\{z_1,\ldots,z_{r'}\}$, such that $\textrm{rank}(JE_Y(\theta)) \geq r'$ for all $\theta \in B'$. Thus, $\theta_0$ is in the real rank stable set for $W$ (and $W \subset Z$), and $|W| = \textrm{dim}_{\textrm{fun}}(\theta_0)$.

 Proposition \ref{localdiffeo} guarantees that $E_W:B' \to \mathbb{R}^{r'}$ gives the first $r'$ coordinates of a diffeomorphic coordinate system on $\Omega$ in an open neighborhood $U$ of $\theta_0$.  In particular, 
 for any $\theta \in U$, 
 the image under $E_W$ of any open neighborhood of $\theta$ in $U$ contains an open neighborhood of $E_W(\theta)$.  Consequently, the set $W$ is persistently pseudo-shattered at $\theta$.  Hence $\psi_Z(\theta) \geq |W|$. 
\end{proof}

\begin{lemma} \label{l:upperbound}
    Let $\mathcal{F}$ be a piecewise-polynomial parameterized family and let $Z = \{z_1,\ldots,z_{|Z|}\} \subset \mathbb{R}^{n_0}$ be a finite set. Let $\theta_0$ be in the algebraically stable set for $Z$.    
If $Z$ is persistently pseudoshattered at $\theta_0$, then $$r_{\mathbb{R}}({\bf J}E_Z^R(\theta)) = |Z|,$$ i.e. the rows of the matrix $\mathbf{J}E_Z^R(\theta_0)$ are  $\mathbb{R}$-linearly independent. 
\end{lemma}

\begin{proof}
    Suppose $Z$ is persistently pseudoshattered at $\theta_0$ but $r_{\mathbb{R}}({\bf J} E_Z^R(\theta)) < |Z|$.  Since $\theta$ is in the algebraically stable set for $Z$, 
    there exists an open neighborhood $U$ of $\theta_0$ on which $E_Z^R$ is constant. 
By assumption, there exist constants $c_2,\ldots,c_{|Z|} \in \mathbb{R}$ such that 
 $$\mathbf{J}E^R_{z_1}\vert_U = \sum_{i=2}^{|Z|} c_i \mathbf{J}E^R_{z_i}\vert_U.$$
Hence (from the fact that a real-valued function $\mathbb{R}^D \to \mathbb{R}$ whose first order partial derivatives are all $0$ on an open set must be constant), there exists $c_{0} \in \mathbb{R}$ such that 
    $$E_{z_1}^R\vert_U = c_0 + \sum_{i=2}^{|Z|} c_i E^R_{z_i}\vert_U.$$
So for any $\theta \in U$
   $$E_{z_1}(\theta)= c_0 + \sum_{i=2}^{|Z|} c_i E_{z_i}(\theta).$$
 Therefore, for any $\theta \in U$,
\begin{equation}\label{eq:linearrelation} 
        E_{z_1}(\theta) - E_{z_1}(\theta_0) = \sum_{i=2}^{|Z|} c_i\left(E_{z_i}(\theta) - E_{z_i}(\theta_0)\right).
        \end{equation}
      In particular, for every $\theta \in U$ such that $$\textrm{sgn}\left((E_{z_i}(\theta) - E_{z_i}(\theta_0)\right) = \textrm{sgn}(c_i)$$ for all $2 \leq i \leq |Z|$,
equation \eqref{eq:linearrelation} implies  
$$\textrm{sgn}\left(E_{z_1}(\theta) - E_{z_1}(\theta_0) \right) \geq 0.$$
Thus, there is no $\theta \in U$ that satisfies 
\begin{align*}
\textrm{sgn}\left((E_{z_i}(\theta) - E_{z_i}(\theta_0)\right)& = \textrm{sgn}(c_i) \,\,\,\,  \forall \,\, 2 \leq i \leq |Z|, \textrm{ and }\\
\textrm{sgn}\left(E_{z_1}(\theta) - E_{z_1}(\theta_0) \right) & = -1,
\end{align*}
contradicting the assumption that $Z$ is pseudoshattered by $\theta_0$. 
\end{proof}

Since the class of ReLU network functions is piecewise polynomial, combining Theorem \ref{t:inequalitystring} with Proposition \ref{p:stabRrksetbig} immediately yields the following corollary.

\begin{corollary}
\label{c:inequalitystring}
Let $\mathcal{F}$ be a parameterized family of ReLU network functions of architecture $(n_0, \ldots, n_d = 1)$.
 Fix a finite set $Z = \{z_1, \ldots, z_m\} \subset \mathbb{R}^{n_0}$. 
Then for Lebesgue almost every $\theta \in \Omega$, 
$$  \left\{\bFD(\theta,Z)  = r_R({\bf J}E^R_Z)\right\}  \leq  \psi_Z(\theta) \leq r_{\mathbb{R}}({\bf J}E_Z^R(\theta)).$$
\end{corollary}

 It is useful to view the gap between the $R$--row rank and the $\mathbb{R}$--row rank of the algebraic Jacobian matrix ${\bf J}E^R_Z(\theta)$ as measuring the potential failure of the functional dimension to agree with $\psi_Z(\theta_0)$, the batch persistent pseudodimension of $\mathcal{F}$ at $\theta_0$ for the batch $Z$.

Note that it is not hard to construct examples of general piecewise-polynomial parameterized families for which $r_R(A) < r_{\mathbb{R}}(A)$, e.g.:

\begin{example} \label{ex:ElisExample}
Let $R = \mathbb{R}[\theta_1,\theta_2,\theta_3]$ and suppose $E_Z^R: \mathbb{R}^3 \rightarrow R^4$ is given by  
\[E^R_Z(\tilde{\theta}) = (\theta_1,\theta_2,\theta_3,\theta_3 - (\theta_1^2 -\theta_2^2))\] for all $\tilde{\theta} \in \mathbb{R}^3$. Then, for all $\tilde{\theta} \in \mathbb{R}^3$, 
\[\mathbf{J}E^R_Z(\tilde{\theta}) =  \left[\begin{array}{ccc} 1 & 0 & 0\\0 & 1 & 0\\0 & 0 & 1\\-2\theta_1 & 2\theta_2 & 1\end{array}\right].\]
Then, for all $\tilde{\theta} \in \mathbb{R}^3$, the $R$-row rank (which equals, for example, the determinantal rank, by Lemma \ref{lem:RankNotationsIntegralDomain}) satisfies 
$$r_R(\mathbf{J}E^R_Z(\tilde{\theta})) = 3$$
while
the $\mathbb{R}$-row rank is
$$r_{\mathbb{R}}(\mathbf{J}E^R_Z(\tilde{\theta})) = 4.$$

Moreover, one can check that all $2^4$ sign sequences are possible for $E_Z(\vec{\theta})$ by choosing suitable values of $\vec{\theta} = (\theta_1, \theta_2, \theta_3)$.
To see this, one need only check that within each octant of $\mathbb{R}^3$ (i.e., for each sign pattern $s \in \{\pm 1\}^3$) there exist $(\theta_1, \theta_2, \theta_3)$ and $(\theta_1',\theta_2',\theta_3'$) with 
\begin{itemize}
    \item $(\mbox{sgn}(\theta_1), \mbox{sgn}(\theta_2), \mbox{sgn}(\theta_3)) = (\mbox{sgn}(\theta'_1), \mbox{sgn}(\theta'_2), \mbox{sgn}(\theta'_3)) = s$,
    \item $\theta_3 > \theta_1^2 - \theta_2^2$, and
    \item $\theta_3' < \theta_1'^2 - \theta_2'^2$.
\end{itemize} 

But regardless of whether $\mbox{sgn}(\theta_3) = \mbox{sgn}(\theta_3') = \pm 1$, choosing $|\theta_1| << |\theta_2|$ and $|\theta'_1| >> |\theta'_2|$ will produce two different signs for $\theta_3 - (\theta_1^2 - \theta_2^2)$ and $\theta_3' - (\theta_1'^2 - \theta_2'^2)$. 
Moreover, all $2^4$ sign patterns can be realized by values of $\vec{\theta}$ within any $\epsilon$ neighborhood of $\vec{0}$.  Since $E_Z(\vec{0}) = \vec{0}$, it follows that $\aleph_4(\vec{0}) = 2^4$, so the set $Z$ is persistently pseudo-shattered at $0$.    
\end{example}

\subsection{The rank gap and overparameterization}

The following classical algebraic result and its corollaries are relevant for supervised learning problems involving piecewise-polynomial parameterized families, in the under- and over-parameterized settings. 

\begin{theorem}[McCoy's Theorem, \cite{McC}] \label{t:McCoy} Let $A$ be an $m \times n$ matrix over a commutative ring $R$. The $R$--module map $R^m \rightarrow R^n$ defined by left multiplication by $R^T$ is injective iff $m \leq n$ and the annihilator of $\mathcal{I}_m$ is $0$.  
\end{theorem}

\begin{corollary} Let $\mathcal{F}$ be a piecewise polynomial parameterized family with parameter space $\mathbb{R}^D$, and let $Z = \{z_1, \ldots, z_m\} \subseteq \mathbb{R}^{n_0}$ be a finite batch of parametrically smooth points. If $m > D$, then $r_R({\bf J}E_Z^R) \leq D < m$. 
\end{corollary}

\begin{corollary} \label{cor:Rlindep}
The rows of an $m \times n$ matrix over $R = \mathbb{R}[\theta_1, \ldots, \theta_D]$ with $m \leq n$ are $R$--linearly independent iff the determinant of at least one $m \times m$ minor of $A$ is a nonzero polynomial.
\end{corollary}

In particular, we note that in the underparameterized case ($D < m$), McCoy's Theorem tells us that $r_R({\bf J}E_Z^R) < m$, the number of rows, so a gap between $r_R({\bf J}E_Z^R)$ and $r_{\mathbb{R}}({\bf J}E_Z^R)$ is possible, and perhaps even likely.

On the other hand, in the overparameterized case $(D > m)$, if it happens that $r_R({\bf J}E_Z^R) = m$, then it is immediate that $r_R({\bf J}E_Z^R) = r_{\mathbb{R}}({\bf J}E_Z^R) = m$.  Note that if $r_R({\bf J}E_Z^R) = r_{\mathbb{R}}({\bf J}E_Z^R)$, there is no rank gap, and hence the inequalities in Theorem \ref{t:inequalitystring} and Corollary \ref{c:inequalitystring} are all equalities, so the batch persistent pseudodimension on $Z$ is equal to the batch functional dimension on $Z$.

It has been hypothesized that in the heavily overparameterized case ($D >> m$), the probability that some $m \times m$ minor of ${\bf J}E_Z^R$ has nonzero determinant (and hence that $r_R({\bf J}E_Z^R) = m$) is very close to $1$ for each $\theta \in \Omega$, cf. the well-conditioning hypothesis of \cite{Belkin}.

\section{Bounds on the persistent pseudodimension for ReLU network families}
\label{sec:mainthm}
We now state and prove the main theorem of the paper.

\begin{theorem} \label{thm:supfundimbounds}
Let $\mathcal{F}$ be a parameterized family of ReLU network functions of architecture $(n_0, \ldots, n_{d-1}|n_d)$ with $n_d = 1$ and parameter space $\Omega = \mathbb{R}^D$. Let $\theta_0 \in \Omega$ be a generic, supertransversal parameter. Let $\bf{Z}$ denote the set of finite subsets of $\mathbb{R}^{n_0}$ that are parametrically smooth with respect to $\theta_0$. Then 
\[\FD(\theta_0) \leq \ppVCD(\theta_0) \leq \sup_{Z \in {\bf Z}} \left\{r_\mathbb{R}({\bf J}E^R_Z)\right\}.\]
\end{theorem}

\begin{proof}
Recall that for any finite set $Z \subset \mathbb{R}^{n_0}$,  $\psi_Z(\theta_0)$ denotes the cardinality of the largest subset of $Z$ that is persistently pseudo-shattered at $\theta_0$. 
Theorem \ref{t:inequalitystring} tells us that for each finite batch $Z$ for which $\theta_0$ is in the algebraically stable set (equivalently, $Z$ is parametrically smooth for $\theta_0$ -- see Remark \ref{rem:algstable} and Lemma \ref{lem:poly}), we have \[\bFD(\theta_0,Z) \leq \psi_Z(\theta_0) \leq r_\mathbb{R}({\bf J}E^R_Z(\theta_0)).\]
Hence
\begin{equation} \label{eq:SupInequalityString} \textrm{sup}_{Z \in \bf Z} \  \bFD(\theta_0,Z) \leq  \textrm{sup}_{Z \in \bf Z} \ \psi_Z(\theta_0) \leq \textrm{sup}_{Z \in \bf Z} \  r_\mathbb{R}({\bf J}E^R_Z(\theta_0)).
\end{equation}

Let ${\bf Z}'$ denote the set of finite subsets of $\mathbb{R}^{n_0}$ (with no additional conditions imposed). By definition, $$\ppVCD(\theta_0) = \sup_{Z \in {\bf Z}'} \{|Z|: Z \textrm{ is persistently pseudo-shattered at } \theta_0\}.$$    Equivalently,
    $$\ppVCD(\theta_0) = \sup_{Z \in {\bf Z}'}\psi_Z(\theta_0).$$
Since $\theta_0$ is assumed to be generic and supertransversal, Corollary \ref{cor:ParSmoothPsi} tells us that  \[\sup_{Z \in {\bf Z}'} \psi_Z(\theta_0) = \sup_{Z \in {\bf Z}} \psi_Z(\theta_0).\] Therefore 
\begin{equation} \label{eq:rewritePPVCD} 
\ppVCD(\theta_0) = \sup_{Z \in {\bf Z}} \psi_Z(\theta_0).  
\end{equation}
By definition,
\begin{equation} \label{eq:fundimdef}
\FD(\theta_0) = \sup_{Z \in {\bf Z}} \bFD(\theta_0,Z).
\end{equation}
Plugging \eqref{eq:fundimdef} and \eqref{eq:rewritePPVCD} into \eqref{eq:SupInequalityString} yields the result. 
\end{proof}

\section{Activation matrix ranks bound Jacobian ranks}
\label{sec:ActMatrix}

In the following, let $\mathcal{F}$ be a parameterized family of ReLU network functions of architecture $(n_0, \ldots, n_{d-1}|n_d)$ with $n_d = 1$ and parameter space $\Omega = \mathbb{R}^D$. Let $\theta_0 \in \Omega$ be a generic, supertransversal parameter, and $Z = \{z_1, \ldots, z_m\}$ be a batch of parametrically smooth points for $\theta_0$.

\begin{definition} \cite[Sec. 2.2]{BP1}
\label{def:activationmatrix}
Let $\mathcal{F}, \theta_0, Z$ be as above. 
\begin{enumerate}[(1)]
\item The {\em algebraic activation matrix}, $\alpha^R(\theta_0,Z)$, of open, complete paths for $Z$ with respect to $\theta_0$ is the matrix with entries in the ring $\mathbb{R}[x_1, \ldots, x_{n_0}]$ whose $m$ rows correspond to $z_1, \ldots, z_m$, whose $|\Gamma|$ columns correspond to $\gamma \in \Gamma$, and for which the entry in the $z_i$ row and $\gamma$ column is:
\begin{itemize}
    \item $x_j$ if $\gamma \in \Gamma^{\theta_0}_{z_i,j}$
    \item $1$ if $\gamma \in \Gamma^{\theta_0}_{z_i,*}$ 
    \item $0$ otherwise
\end{itemize} 
\item The {\em real activation matrix}, $\alpha(\theta_0,Z)$, is the matrix with entries in $\mathbb{R}$ obtained by evaluating the $m$ rows of $\alpha^R(\theta_0, Z)$ at $z_1, \ldots, z_m$, respectively. That is, if we vectorize the notation in Equation \ref{eq:substitutex} in the standard way, \[\alpha(\theta_0,Z) = \alpha^R(\theta_0,Z)/(X \rightarrow Z).\]
\end{enumerate}
\end{definition}

\begin{remark}
$\alpha^R(\theta_0,Z)$ and $\alpha(\theta_0,Z)$ are closely related to the {\em activation operators} defined in \cite{BP1}, so we match their notation and terminology. Many of the ideas that make both the algebraic and real versions of the activation matrix natural objects of study also appear in \cite{GS}. 
\end{remark}

The following result, essentially stated and proved in \cite{BP1} (see also \cite{GS}),\footnote{\cite{BP1,GS} do not explicitly adopt the perspective of considering matrices over polynomial rings, but these ideas are at the heart of what they do.} allows for the following convenient means of calculating the polynomials associated to the $\pm$ activation regions of any generic, supertransversal $\theta_0 \in \Omega$.

\begin{lemma}
\label{l:actprod}
Let $\mathcal{F}, \theta_0, Z$ be as above, and let ${\bf \Gamma}$ be the $|\Gamma| \times 1$ matrix over $\mathbb{R}[\theta_1, \ldots, \theta_D]$ whose entry in the column corresponding to a path $\gamma$ is $m(\gamma)$. Then viewing both $\alpha^R(\theta_0,Z)$ and ${\bf \Gamma}$ as matrices over $\mathbb{R}[\theta_1, \ldots, \theta_D,x_1, \ldots, x_{n_0}]$, the $i$th row of the matrix product $\alpha^R(\theta_0, Z){\bf \Gamma}$ is $P(\theta_0, z_i)$, and the $i$th row of the matrix product $\left\{\alpha^R(\theta_0,Z){\bf \Gamma}\right\}/(X \rightarrow Z)$ is $E^R_Z(\theta_0)$.   
\end{lemma}

\begin{proof}
    Immediate from the definitions, once we restrict to $\theta_0$ being generic and supertransversal and $Z$ being a batch of parametrically smooth points for $\theta_0$, which forces each $z_i$ to be in a $\pm$ activation region for $\theta_0$, guaranteeing that $P(\theta_0, z_i)$ is well-defined for all $i$. Since the $i$th row of $\alpha^R(\theta_0, Z){\bf \Gamma}$ is $P(\theta_0, z_i)$, its evaluation at $z_i$ is precisely $E^R_{z_i} := P(\theta_0,z_i)/(x \rightarrow z_i)$. See also \cite[Prop.1]{BP1} and \cite[Sec. 4]{GS}.
\end{proof}

We are now ready to prove the main result of this section.

\begin{proposition}
\label{p:multirankupperbound} Let $\mathcal{F}, \theta_0, Z$ be as above. Then we have the following upper bound on the $\mathbb{R}$--row rank of the algebraic Jacobian matrix, ${\bf J}E^R_Z$:
\begin{eqnarray}
\label{eq:Rankineq}
r_{\mathbb{R}}({\bf J}E^R_Z) &\leq& \mbox{Rank}(\alpha(\theta_0,Z)).
\end{eqnarray}
\end{proposition}

\begin{proof}
First note that Lemma \ref{lem:poly} tells us that for all $i = 1, \ldots, m$ and all $j = 1, \ldots, n_0$, denoting the $j$th coordinate of $z_i$ by $z_{i,j}$ we have: 
\[
P(\theta_0,z_i)/(x \rightarrow z_i) =  z_{i,j}\sum_{\gamma \in \Gamma^{\theta_0}_{z_i,j}} m(\gamma) + \sum_{\gamma \in \Gamma^{\theta_0}_{z_i,*}}m(\gamma).\] Then by Lemma \ref{l:actprod} and the fact that ${\bf \Gamma}$ has no $x$ dependence, we see that \[\left\{\alpha^R(\theta_0,Z){\bf \Gamma}\right\}/(X \rightarrow Z) = \alpha(\theta_0,Z){\bf \Gamma}.\]

Now suppose $\mbox{Rank}(\alpha(\theta_0,Z)) = k$ as a matrix over $\mathbb{R}$. Then we may assume without loss of generality (by permuting the rows if necessary) that the first $k$ rows, $R_1, \ldots, R_k$, of $\alpha(\theta_0,Z)$ are linearly independent and that for each $i = k+1, \ldots, m$, there exist $c_{i,j} \in \mathbb{R}$ such that 
\begin{equation}
    \label{eq:rowdepend}
    R_{i} = \sum_{j=1}^k c_{i,j} R_j
\end{equation} 

But this immediately implies that for all $k+1\leq i \leq n$, we have \begin{equation} \label{eq:polylinrelation}
E^R_{z_i}(\theta_0) = \sum_{j=1}^k c_{i,j} E^R_{z_j}(\theta_0),
\end{equation}

since if $C$ is the $m \times m$ lower triangular matrix over $\mathbb{R}$ with 
\begin{itemize}
    \item $1's$ on the diagonal,
    \item $-c_{i,j}$ in the $i$th row and $j$th column when $k+1 \leq i \leq m$ and $1 \leq j \leq k$,
    \item $0$ otherwise,
\end{itemize}  then rows $k+1, \ldots, m$ of the matrix product $C \alpha(\theta_0,Z)$ are all ${\bf 0} \in \mathbb{R}^{|\Gamma|}$  by Equation \ref{eq:rowdepend}, which implies that rows $k+1, \ldots, m$ of \[(C\alpha(\theta_0,Z)){\bf \Gamma} = C(\alpha(\theta_0,Z){\bf \Gamma})\] are ${\bf 0}$ as well. Note here that although $C$ and $\alpha(\theta_0,Z)$ have entries in $\mathbb{R}$, we are regarding them  in the product above as matrices over $R = \mathbb{R}[\theta_1, \ldots, \theta_D]$, and using the fact that if $A \in M_{m\times n}(R)$ and $B \in M_{n\times p}(R)$ are matrices over a commutative ring $R$, and the $i$th row of $A$ is ${\bf 0} \in R^n$, then the $i$th row of $AB$ is ${\bf 0} \in R^p$.

Now for each $k+1 \leq i \leq m$, the linearity of partial derivatives tells us that the formal partial derivatives of $E^R_{z_i}$ with respect to the parameters $\theta$ satisfy the same $\mathbb{R}$--linear dependence relation as in Equation \ref{eq:polylinrelation}. That is, for each weight $w$ (component of the weight matrix $W^\ell$) associated to each layer map $F^\ell$ we have:

\[ \frac{\partial E^R_{z_i}}{\partial w}  = \sum_{j=1}^k c_{i,j}\frac{\partial E^R_{z_j}(\theta_0)}{\partial w},\]

and for each bias $b$ associated to each layer map $F^\ell$, we also have: 
\[ \frac{\partial E^R_{z_i}(\theta_0)}{\partial b}  = \sum_{j=1}^k c_{i,j}\frac{\partial E^R_{z_j}(\theta_0)}{\partial b},\]

which tells us that the formal gradient vectors also satisfy the same linear dependence relation over $\mathbb{R}$:
\[\nabla_\theta E^R_{z_i}(\theta_0) = \sum_{j=1}^k c_{i,j}\nabla_\theta E^R_{z_j}(\theta_0).\] But by definition, $\nabla_\theta E^R_{z_i}(\theta_0)$ is the $i$th row of ${\bf J}E^R_Z$, which tells us that the $j$th row of the matrix product $C{\bf J}E^R_Z(\theta_0)$ is ${\bf 0} \in R^{|\Gamma|}$ for $j = k+1, \ldots, m$. 

We conclude that \[r_{\mathbb{R}}({\bf J}E^R_Z) \leq k = \mbox{rank}(\alpha(\theta_0,Z)),\] as desired.    
\end{proof}

We end this section by noting that although we don't have enough concrete evidence to state the following as a conjecture, it is hard to imagine how one might obtain $R$--linear dependence relations among the rows of ${\bf J}E_Z$ other than via the mechanism described above, since the polynomials are constructed as sums of paths, and the derivatives also have interpretations via sums of paths. We therefore ask:

\begin{question}
\label{q:OtherBound}
With notation and assumptions as in Proposition \ref{p:multirankupperbound}, is it the case that for all parametrically smooth batches $Z$ for generic, supertransversal parameters $\theta_0 \in \Omega$, we have \[\mbox{Rank}(\alpha(\theta_0,Z)) \leq r_R({\bf J}E^R_Z)?\] 
\end{question} 

Combining Proposition \ref{p:multirankupperbound} with Theorem \ref{t:inequalitystring}, we note that if the answer to Question \ref{q:OtherBound} is yes, then all of the inequalities in Theorem \ref{t:inequalitystring} will be forced to be equalities. We will go ahead and conjecture:

\begin{conjecture}
\label{c:pPdimequalsfundim} Let $\mathcal{F}$ be a parameterized family of ReLU network functions of architecture $(n_0, \ldots, n_{d-1}|1)$ and parameter space $\Omega = \mathbb{R}^D$. Let $\theta_0 \in \Omega$ be a generic, supertransversal parameter, and $Z = \{z_1, \ldots, z_m\}$ be a batch of parametrically smooth points for $\theta_0$. Then

\[\mbox{dim}_{fun}(\theta_0) = \ppVCD(\mathcal{F},\theta_0)\]
\end{conjecture}

\bibliography{dimensionbibliography}

\end{document}